\documentclass{article}

\usepackage{microtype}
\usepackage{graphicx}
\usepackage{booktabs} 

\usepackage{hyperref}
\usepackage{apxproof}
\usepackage{caption}
\usepackage{subcaption}

\usepackage[accepted]{icml2021}

\icmltitlerunning{Value Alignment Verification}

\usepackage[utf8]{inputenc} 
\usepackage[T1]{fontenc}    
\usepackage{url}            
\usepackage{amsfonts}       
\usepackage{nicefrac}       
\usepackage{amsmath}
\usepackage{amssymb}
\usepackage{amsthm}
\newtheorem{theorem}{Theorem}
\newtheorem{lemma}{Lemma}

\newtheorem{corollary}{Corollary}
\newtheorem{definition}{Definition}
\newtheorem{proposition}{Proposition}
\newcommand{\wb}{\mathbf{w}}
\usepackage{subcaption}
\usepackage{wrapfig}
\usepackage{xr}

\DeclareMathOperator{\support}{support}

\begin{document}

\twocolumn[

\icmltitle{Value Alignment Verification}

\icmlsetsymbol{equal}{*}

\begin{icmlauthorlist}
\icmlauthor{Daniel S. Brown}{equal,berkeley}
\icmlauthor{Jordan Schneider}{equal,austin}
\icmlauthor{Anca Dragan}{berkeley}
\icmlauthor{Scott Niekum}{austin}
\end{icmlauthorlist}

\icmlaffiliation{berkeley}{University of California, Berkeley, USA}
\icmlaffiliation{austin}{University of Texas at Austin, USA}

\icmlcorrespondingauthor{Daniel Brown}{dsbrown@berkeley.edu}
\icmlcorrespondingauthor{Jordan Schneider}{joschnei@cs.utexas.edu}

\icmlkeywords{Value Alignment, AI Safety, Reinforcement Learning}

\vskip 0.3in
]

\printAffiliationsAndNotice{\icmlEqualContribution}

\begin{abstract}
As humans interact with autonomous agents to perform increasingly complicated, potentially risky tasks, it is important to be able to efficiently evaluate an agent's performance and correctness. In this paper we formalize and theoretically analyze the problem of efficient \textit{value alignment verification}: how to efficiently test whether the behavior
of another agent is aligned with a human's values. The goal is to construct a kind of ``driver's test” that a human can give to any agent which will verify value alignment via a minimal number of queries.
We study alignment verification problems with both idealized humans that have an explicit reward function as well as problems where they have implicit values. We analyze verification of exact value alignment for rational agents and propose and analyze heuristic and approximate value alignment verification tests in a wide range of gridworlds and a continuous autonomous driving domain. Finally, we prove that there exist sufficient conditions such that we can verify exact and approximate alignment across an infinite set of test environments via a constant-query-complexity alignment test.
\end{abstract}

\section{Introduction}
If we desire autonomous agents that can interact with and assist humans and other agents in performing complex, risky tasks, then it is important that humans can verify that these agents' policies are aligned with what is expected and desired. This alignment is often termed \textit{value alignment} and is defined in the Asilomar AI Principles\footnote{
\url{https://futureoflife.org/ai-principles/} } as follows:
"Highly autonomous AI systems should be designed so that their goals and behaviors can be assured to align with human values throughout their operation."
In this paper, we provide a theoretical analysis of the problem of \textbf{efficient value alignment verification}: \textit{how to efficiently test whether a robot is aligned with a human's values.}

Existing work on value alignment often focuses on qualitative evaluation of trust~\cite{huang2018establishing} or asymptotic alignment of an agent's performance via interactions and active learning~\cite{hadfield2016cooperative,christiano2017deep,sadigh2017active}. By contrast, our work analyzes the difficulty of efficiently evaluating another agent's correctness by formally defining value alignment and seeking efficient tests for value alignment verification that are applicable when two or more agents already have learned a policy or reward function and want to efficiently test compatibility. 
To the best of our knowledge, we are the first to define and analyze the problem of value alignment verification. In particular, we propose exact, approximate, and heuristic tests that one agent can use to quickly and efficiently verify value alignment with another agent. 

As depicted in Figure~\ref{fig:huma_testing_robot_or_human}, the goal of value alignment verification is to construct a kind of “driver's test” that a human can give to any agent which will verify value alignment and consists of only a small number of queries. We define values in the reinforcement learning sense, i.e., with respect to a reward function: a robot is exactly value aligned with a human if the robot's policy is optimal under the human's reward function.
The two agents in a value alignment verification problem (human and robot) may have different communication mechanisms and different value introspection abilities. Thus, the way we analyze value alignment verification will depend on 
whether the human's and robot's access to their values is \textit{explicit}, e.g., able to write down a value function or reward function or \textit{implicit}, e.g., able to answer preference queries or sample actions from a policy.
The most general version of value alignment verification involves a human with implicit values who seeks to verify the value alignment of a robot with implicit values, e.g. a black-box policy. This setting motivates our work;
however, it is challenging and we postpone many questions for future research.

We follow a ground-up approach where we analyze the difficulty of value alignment verification starting in the most idealized setting, and then gradually relax our assumptions. We first analyze sufficient conditions under which efficient exact value alignment verification is possible in the \textit{explicit human, explicit robot} setting, where an idealized human tester knows their reward function and so does the robot. When the robot is rational with respect to a reward function that is a linear combination of known features, we show that it is possible to provably verify the alignment of any rational explicit robot via a succinct test consisting of either reward queries, value queries, or trajectory preference queries.  
We next consider the \textit{explicit human, implicit robot} setting, where an idealized human knows their reward function, but seeks to efficiently verify the alignment of a black-box policy via action queries. We study heuristics for generating value alignment verification tests in this setting and compare their performance on a range of gridworlds. 

Finally, in Section~\ref{sec:human} we study the most general setting of \textit{implicit human, implicit robot}. We propose an algorithm for approximate value alignment verification in continuous state and action spaces and provide empirical results in a continuous autonomous driving domain where the human can only query the robot for preferences over trajectories. 
We conclude with a brief discussion of the challenge of designing value alignment verification tests that generalize across multiple MDPs. Somewhat surprisingly, we provide initial theory demonstrating that if the human can create the test environment for the robot, then exact and approximate value alignment across an infinite family of MDPs can be verified by observing the robot's policy in only two carefully constructed test environments.

Source code and videos are available at \url{https://sites.google.com/view/icml-vav}.

\begin{figure*}[th]
    \centering
    \includegraphics[width=0.8\linewidth]{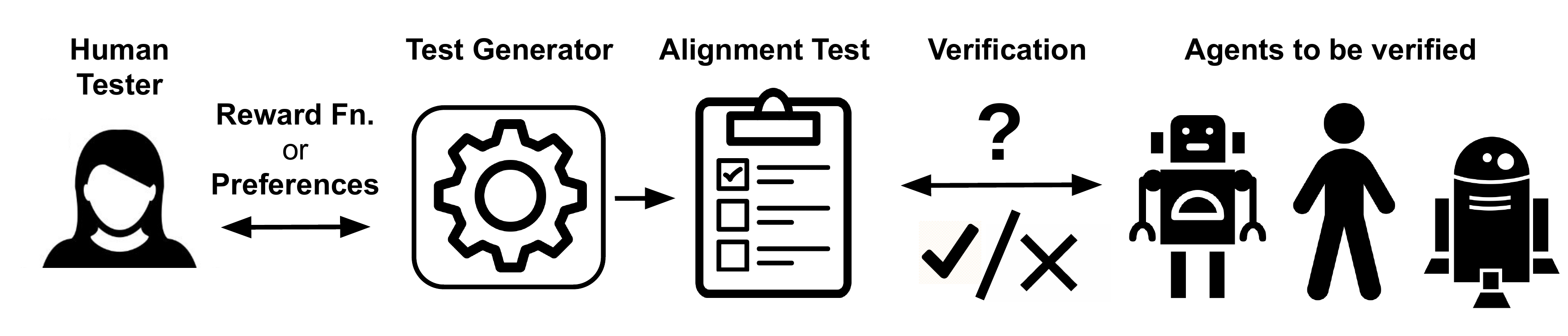}
    \caption{The tester provides a reward function either explicitly or implicitly to a test generation algorithm which distills the human's values into a succinct alignment test. This single test is used to efficiently verify the value alignment of any agent.
    }
    \label{fig:huma_testing_robot_or_human}
\end{figure*}

\section{Related work} 

\paragraph{Value Alignment:} Most work on value alignment focuses on how to iteratively train a learning agent such that its final behavior is aligned with a user's intentions~\cite{leike2018scalable, russell2015research,amodei2016concrete}.
One example is cooperative inverse reinforcement learning (CIRL)~\cite{hadfield2016cooperative,fisac2020pragmatic,shahbenefits}, which formulates value alignment as a game between a human and a robot, where both try to maximize a shared reward function that is only known by the human. CIRL and other research on value alignment focus on ensuring the learning agent asymptotically converges to the same values as the human teacher, but do not provide a way to check whether value alignment has been achieved. By contrast, we are interested in value alignment \textit{verification}. Rather than assuming a cooperative setting, we assume the robot being tested has already learned a policy or reward function and the human wants to efficiently verify whether the robot is value aligned.

\paragraph{Reward Learning:}
Inverse reinforcement learning (IRL)~\cite{ng2000algorithms,abbeel2004apprenticeship,arora2018survey} and active preference learning~\cite{wirth2017survey,christiano2017deep,biyik2019asking} algorithms aim to determine the reward function of a human via offline demonstrations or online queries. In contrast, value alignment verification only seeks to answer the question of whether two agents are aligned, without concern for the exact reward function of the robot. In Section~\ref{sec:omnipotent_vav} we prove that value alignment verification can be performed in a constant number of queries whereas active reward learning requires a logarithmic number of queries~\cite{amin2016towards,amin2017repeated}. 
In cases where the human has implicit values, active reward learning can be used to infer the reward function of the human tester, and then this inferred reward function can be used to automatically generate a high-confidence value alignment test. While active reward learning may be a subcomponent of value alignment verification, it focuses on customizing reward inference queries for a single agent, whereas value alignment verification seeks to design a single alignment test that works for all agents.

\paragraph{Machine Teaching:}
In machine teaching~\cite{zhu2018overview}, a teacher seeks to optimize a minimal set of training data such that a student (running a particular learning algorithm) learns a desired set of model parameters. Value alignment verification can be seen as a form of machine \textit{testing} rather than teaching---machine teaching algorithms typically search for a minimal set of training data that will teach a learner a specific model, whereas we seek a minimal set of questions that will allow a tester to verify whether another agent's learned model is correct. Thus, in machine teaching, the teacher provides examples and their answers, but in machine testing the tester provides examples and then queries the testee for the answer. While machine teaching has been applied to sequential decision making domains~\cite{cakmak2012algorithmic,huang2017enabling,brown2019machine}, we are not aware of any work that considers the problem of value alignment verification. 

\paragraph{Policy Evaluation}
Policy evaluation~\cite{sutton1998introduction} aims to answer the question, "How much return would another agent achieve according to my values?" By focusing on the simpler decision problem, "Is the robot value aligned with the human?", we seek tests that are much more sample-efficient than running a full policy evaluation. Off-Policy Evaluation (OPE) seeks to perform policy evaluation without executing the testee's policy~\cite{precup2000eligibility,thomas2015high,hanna2017bootstrapping}. However, OPE is often sample-inefficient, provides high-variance estimates, and typically assumes explicit access to the tester's reward function, and the tester and testee policies. Value alignment verification is applicable in settings where the policies and reward functions of both agents may be implicit and only accessible indirectly.

\section{Notation}

We adopt notation proposed by Amin et al.~\cite{amin2017repeated} where a Markov Decision Process (MDP) $M$ consists of an environment $E = (\mathcal{S}, \mathcal{A}, P, S_0, \gamma)$ and a reward function $R: \mathcal{S} \to \mathbb{R}$. An environment $E$, consists of a set of states $\mathcal{S}$, a set of actions $\mathcal{A}$, a transition function $P:\mathcal{S} \times \mathcal{A} \times \mathcal{S} \to [0,1]$ from state-action pairs to a distribution over next states, a discount factor $\gamma \in [0,1)$, and a distribution over initial states $S_0$. A policy $\pi : \mathcal{S} \times \mathcal{A} \to [0,1]$ is a mapping from states to a distribution over actions.
The state and state-action values of a policy $\pi$ are 
$V^\pi_{R}(s) = \mathbb{E}_{\pi}[\sum_{t=0}^\infty \gamma^t R(s_t) \mid s_0 = s]$ and
$Q^{\pi}_{R}(s,a) = \mathbb{E}_{\pi} [\sum_{t=0}^\infty\gamma^t R(s_t) \mid s_0 = s, a_0 = a]$ for $s \in \mathcal{S}$ and $a \in \mathcal{A}$. We denote $V^*_{R}(s) = \max_{\pi} V^{\pi}_{R}(s)$ and $Q^*_{R}(s,a) = \max_{\pi} Q^{\pi}_{R}(s,a)$. The expected value of a policy is denoted by $V^\pi_{R} = \mathbb{E}_{s \in S_0} [V_{R}^\pi(s)]$. 

We assume that the $\arg\max$ operator returns a set, i.e., $\arg \max_x f(x) := \{x \mid f(y) \leq f(x), \forall y \}$. We let $\pi^*_R \in \arg\max_\pi V^\pi_{R}$ denote an optimal policy under reward function $R$. We also let $\mathcal{A}_R(s) = \arg\max_{a'\in \mathcal{A}} Q_{R}^*(s,a')$ denote the set of all optimal actions at state $s$ under reward function $R$. Thus, $\mathcal{A}_R(s) = \{ a \in \mathcal{A} \mid \pi^*_R(a|s) > 0  \}$

As is common~\cite{ziebart2008maximum,barreto2017successor,brown2020safe}, we assume that the reward function is linear under features $\phi: \mathcal{S} \mapsto \mathbb{R}^k$, so that $R(s) = \mathbf{w}^T \phi(s)$, where $\mathbf{w} \in \mathbb{R}^k$. Thus, we use $R$ and $\mathbf{w}$ interchangeably.
Note that this assumption of a linear reward function is not restrictive as these features can be arbitrarily complex nonlinear functions of the state and could be obtained via unsupervised learning from raw state observations~\cite{laskin2020curl,brown2020safe}. Given that $R(s) = \mathbf{w}^T \phi(s)$, the state-action value function can be written in terms of discounted expectations over features~\cite{abbeel2004apprenticeship}:
$Q^{\pi}_{R}(s,a) = \mathbf{w}^T \Phi_{\pi}^{(s,a)}$,
where 
$\Phi_{\pi}^{(s,a)} = \mathbb{E}_{\pi} [\sum_{t=0}^\infty\gamma^t \phi(s_t) \mid s_0 = s, a_0 = a]$.

\section{Value Alignment Verification}\label{sec:vav}
In this section we first explicitly define value alignment and value alignment verification. Next, we discuss how assuming rationality of the robot enables efficient provable value alignment verification.
We then examine how to perform (approximate) value alignment verification in tabular MDPs under different forms of test queries, including reward, value, preference, and action queries. We conclude this section by presenting a method for approximate value alignment verification when the tester is a human with implicit values and the state and action spaces are continuous.

We first formalize value alignment. 
Consider two agents: a human and a robot. We will assume that the human has a (possibly implicit) reward function that provides the ground truth for determining value alignment verification of the robot. We define (approximate) value alignment as follows:
\begin{definition}\label{def:eps_va}
Given reward function $R$, policy $\pi'$ is $\epsilon$-\textbf{value aligned} in environment $E$ if and only if
\begin{equation}
V^{*}_{R}(s) - V^{\pi'}_{R}(s)  \leq \epsilon,  \forall s \in \mathcal{S}.
\end{equation}
\end{definition}
Exact value alignment is achieved when $\epsilon=0$.

We are interested in \textbf{efficient value alignment verification} where we can correctly classify agents as aligned or misaligned within certain error tolerances while keeping the total test size small. Formally, efficient  (approximate) value alignment verification is a solution to the following:
\begin{align} \label{eq:evav_problem}
    &\min_{T \subseteq \mathcal{T}} |T|, \text{s.t.}\; \forall \pi' \in \Pi, \forall s \in \mathcal{S} \\
    &V^{*}_{R}(s) - V^{\pi'}_{R}(s)  > \epsilon \Rightarrow Pr[\text{$\pi'$ passes test $T$}] \leq \delta_{\rm fpr} \nonumber \\
    &V^{*}_{R}(s) - V^{\pi'}_{R}(s)  \leq \epsilon \Rightarrow Pr[\text{$\pi'$ fails test $T$}] \leq \delta_{\rm fnr} \nonumber
\end{align}
where $\mathcal{T}$ is the choice set of possible test queries, $\Pi$ denotes the set of robot policies for which we design the test, $\delta_{\rm fp}, \delta_{\rm fn} \in [0,1]$ denote the allowable false positive rate and false negative rate, and $|T|$ denotes the cardinality, or complexity of the test, $T$. If $\epsilon = \delta_{\rm fpr}= 0$, then we seek the test that enables exact value alignment verification. 

\subsection{Query Types and Rational Agents}
The difficulty of solving Equation~\ref{eq:evav_problem} can change significantly as a function of $\epsilon$, $\delta_{\rm fpr}$, $\delta_{\rm fnr}$, the set of policies for which we design the test $\Pi$, and the type of queries available in the choice set $\mathcal{T}$.
For example, exact alignment is impossible in settings where one can only query for actions (see Appendix~\ref{app:bb_impossible}). Even when possible, achieving high confidence may require multiple action queries at every state.

One of the main goals of this paper is to understand under what settings we can achieve efficient, provable value alignment verification. Towards this end, we assume that the robot behaves rationally with respect to some reward function $R'$.
A \textit{rational agent} is one that picks actions to maximize its utility~\cite{russell2016artificial}. Formally $\pi'$ is a rational agent if:
\begin{equation}\label{def:rational_agent}
\forall a \in \mathcal{A}, \pi'(a|s) > 0 \implies a \in \arg \max_a Q^{*}_{R'}(s,a),
\end{equation}
where $\arg \max_a Q^{*}_{R'}(s,a)$ returns the set of all optimal actions at state $s$ under $R'$.

Note that rationality in itself does not restrict the set of policies $\Pi$ for which we can test, since all policies are rational under the trivial all zero reward function~\cite{ng2000algorithms}. Rationality also does not limit the choice set $\mathcal{T}$ since a rational agent can answer any question related to its policy or values.
The rationality assumption is helpful because it directly connects the behavior of the agent to a reward function: given behavior we can infer rewards and given rewards we can infer behavior. It also allows us to extrapolate robot behavior to new situations, enabling efficient value alignment verification.

\subsection{Exact Value Alignment}\label{sec:robot_robot_vav}
We start with the idealized query setting of \textit{explicit human, explicit robot}.
In this section we discuss exact value alignment ($\epsilon = 0 ,\delta_{\rm fpr}=0$) of a rational robot and review related work by \citep{ng2000algorithms} on sets of rewards consistent with an optimal policy. Then in the next section we will examine how to construct verification tests for exact alignment.
We assume that both the human and robot know the states reward features $\phi(s)$, and that the robot acts rationally with respect to a reward function linear in these features. 

Consider two rational agents with reward functions $R$ and $R'$. Because there are infinite reward functions that lead to the same optimal policy~\cite{ng2000algorithms}, determining that $\exists s \in S, R(s) \neq R'(s)$ does not necessarily imply misalignment.  
For ease of notation, we define 
\begin{equation}
    OPT(R) = \{\pi \mid \pi(a|s) > 0 \Rightarrow a \in \arg\max_a Q^*_{R}(s,a) \}, \nonumber
\end{equation} 
as the set of all optimal (potentially stochastic) policies in MDP $(E,R)$.
Combining Definition~\eqref{def:eps_va} and Equation~\eqref{def:rational_agent} immediately gives us that a rational robot is aligned with a human if all optimal policies under the robot's reward function are also optimal policies under the human's reward function. We formally state this as the following Corollary.
\begin{corollary}\label{lem:va_explicit}
We have \textbf{exact value alignment} in environment $E$ between a rational robot with reward function $R'$ and a human with reward function $R$ if
$OPT(R') \subseteq OPT(R)$.
\end{corollary}

We now review foundational work on IRL by Ng and Russell~\cite{ng2000algorithms} which inspires our proposed approach for efficient value alignment verification.
\begin{definition}
Given an environment $E$, the \textbf{consistent reward set} (CRS) of a policy $\pi$ in environment $E$ is defined as the set of reward functions under which $\pi$ is optimal:
\begin{equation}\label{eq:CRS}
\text{CRS}(\pi) =  
\{R \mid \pi \in OPT(R) \}. \end{equation}
\end{definition}
When $R(s) = \wb^T \phi(s)$, the CRS is the following polytope:
\begin{corollary} \label{corr:cont_crs}~\cite{ng2000algorithms,brown2019machine}
Given an environment $E$, the $CRS(\pi)$ is given by the following intersection of half-spaces:
\begin{equation}
\begin{aligned}
\{\mathbf{w} \in \mathbb{R}^k \mid \mathbf{w}^T (\Phi_{\pi}^{(s,a)} - \Phi_{\pi}^{(s,b)}) \geq 0, \\
\forall a \in \arg\max_{a' \in \mathcal{A}} Q^\pi_R(s,a'), b \in \mathcal{A}, s \in \mathcal{S} \}.
\end{aligned}
\end{equation}
\end{corollary}

\begin{figure}
     \centering
     \begin{subfigure}[b]{0.34\linewidth}
         \centering
         \includegraphics[width=\textwidth]{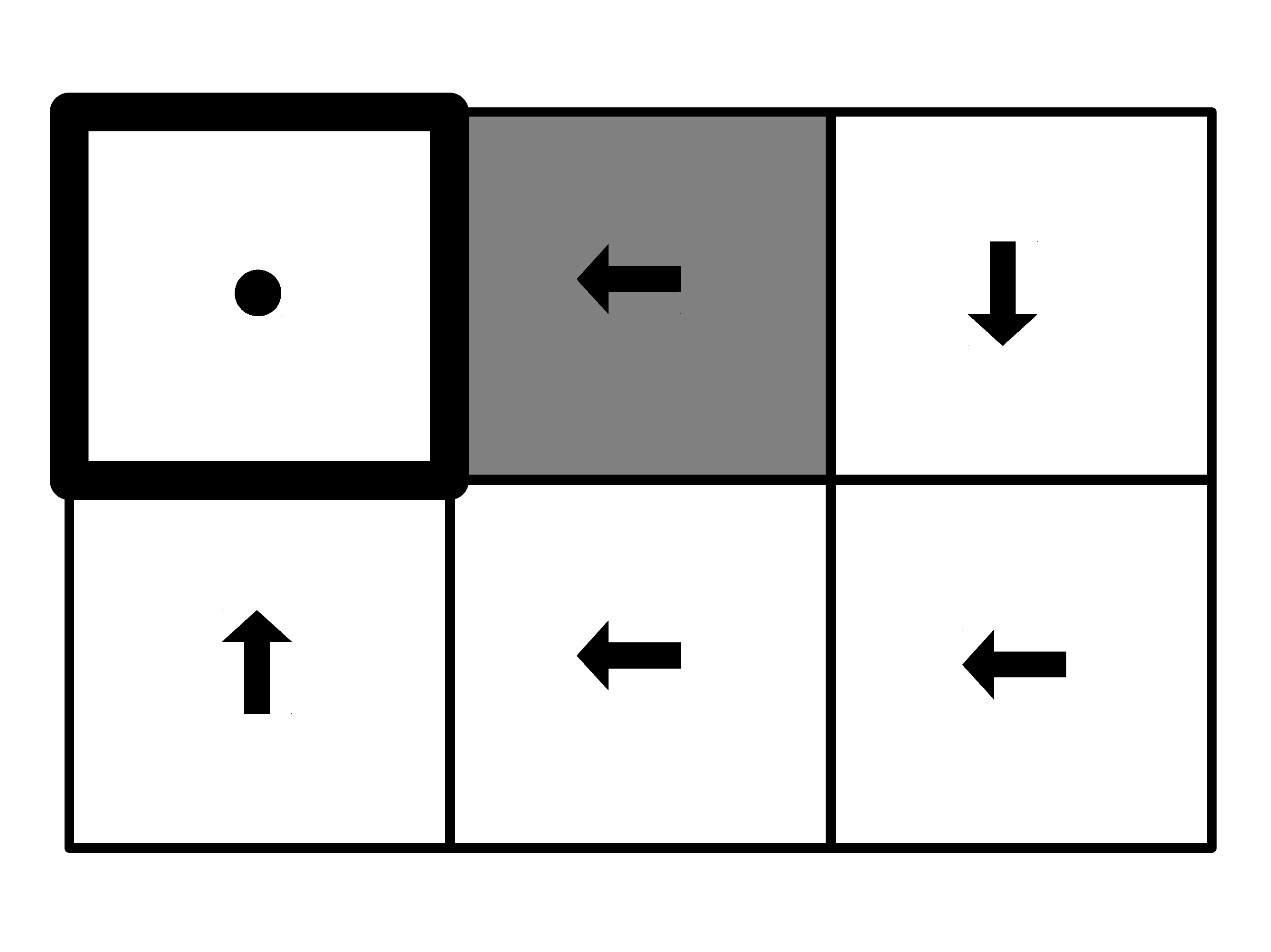}
         \caption{Policy $\pi$}
         \label{fig:crs_policy}
     \end{subfigure}
     \begin{subfigure}[b]{0.48\linewidth}
         \centering
         \includegraphics[width=\textwidth]{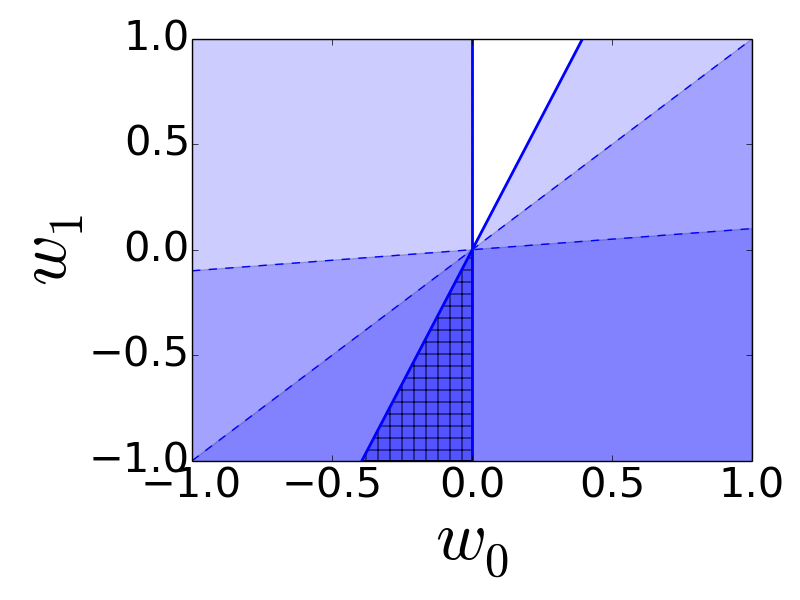}
         \caption{$CRS(\pi)$.}
         \label{fig:crs_halfspaces}
     \end{subfigure}
        \caption{An example of the consistent reward set (CRS) for a policy $\pi$ in a simple gridworld and a linear reward function with two binary reward features (white and gray) with reward weights $w_0$ and $w_1$, respectively.}
        \label{fig:crs_example}
\end{figure}

As an example consider the grid world MDP shown in Figure~\ref{fig:crs_example}. The CRS is an intersection of half-spaces which define all reward functions under which $\pi$ is optimal. Note that the all zero reward function and the reward function where white cells have zero reward are included; however, not all optimal policies under these reward functions lead to the policy shown in Figure~\ref{fig:crs_policy}.

Thus, we cannot directly use Corollary~\ref{corr:cont_crs} to verify alignment with a human's optimal policy---Corollary~\ref{corr:cont_crs} only provides a necessary, but not sufficient, condition for testing whether a reward function $R'$ is value aligned with a policy $\pi$. Consider the example of the trivial all zero reward function: it is always in the CRS of any policy; however, an agent optimizing the zero reward can result in any arbitrary policy. Even ignoring the all zero reward, rewards can be on the boundaries of the CRS polytope that are consistent with a policy, but not value aligned since they lead to more than one optimal policy, one or more of which may not be optimal under the human's reward function. In the next section we show that if we remove all such edge cases, we can construct an \textit{aligned reward polytope} (ARP) similar to the CRS, which enables provable value alignment verification. Furthermore, we show that the aligned reward polytope can be used for alignment verification even when the human cannot directly query for the robot's reward function.

\subsection{Sufficient Conditions for Provable Verification of Exact Value Alignment}\label{sec:exact_alignment_algos}

We seek an efficient value alignment verification test which enables a human to query the robot to determine exact value alignment as in Corollary~\ref{lem:va_explicit}. 
The following theorem demonstrates that provable verification of exact value alignment is possible under a variety of query types.

\begin{theorem}\label{thm:rational_vav_equivalence}
Under the assumption of a rational robot that shares linear reward features with the human, efficient exact  value alignment verification is possible in the following query settings: (1) Query access to reward function weights $\mathbf{w}'$, (2) Query access to samples of the reward function $R'(s)$, (3) Query access to $V^*_{R'}(s)$ and $Q^*_{R'}(s,a)$, and (4) Query access to preferences over trajectories.
\end{theorem}

\subsubsection{Case 1: Reward Weight Queries}
We first consider the case where the human can directly query the robot for their reward function weights $\wb'$. While this problem setting is mainly of theoretical interest, we will show that Cases (2) and (3) also reduce to this setting. Querying directly for the robot's reward function is maximally efficient since by definition it only requires a single query.
Although one can solve for the optimal policy under a given $\wb'$ and evaluate it under the human's reward function $\wb$, this brute force approach is computationally demanding and must be repeated for each robot that needs to be tested.
By contrast, we will prove that there exists a single efficient verification test that does not require solving for the robot's optimal policy and can be used to verify the alignment of any robot.

As mentioned in the previous section, the CRS for the human's optimal policy does not provide a sufficient test for value alignment verification. Under the assumption of a rational robot, a sufficient condition for value alignment verification is to test whether a robot's reward function lies in the following set:
\begin{definition}\label{def:ARP}
Given an MDP $M$ composed of environment $E$ and reward function $R$, the \textbf{aligned reward set} (ARS) is defined as the following set of reward functions:
\begin{align}\label{eq:aec}
\text{ARS}(R) &=  
\{R' \mid \text{OPT}(R') \subseteq \text{OPT}(R) \}.
\end{align}
\end{definition}
Using Definition~\ref{def:ARP}, we prove the following lemma which will enable efficient verification of exact value alignment. As a reminder, we use the notation $Q^{\pi}_{R}(s,a) = \mathbf{w}^T \Phi_{\pi}^{(s,a)}$,
for 
$\Phi_{\pi}^{(s,a)} = \mathbb{E}_{\pi} [\sum_{t=0}^\infty\gamma^t \phi(s_t) \mid s_0 = s, a_0 = a]$, and
$\mathcal{A}_R(s) = \arg\max_{a'\in \mathcal{A}} Q_{R}^*(s,a')$.

\begin{lemma}\label{lem:arp_va}
Given an MDP $M = (E,R)$, assuming the human's reward function $R$, and the robot's reward function $R'$ can be represented as linear combinations of features $\phi(s) \in \mathbb{R}^k$, i.e., $R(s) = \wb^T \phi(s)$, $R'(s) = {\wb'}^T \phi(s)$, and given an optimal policy $\pi_R^*$ under $R$ then
\begin{equation}
\wb' \in \bigcap_{(s,a,b) \in \mathcal{O}} \mathcal{H}^{R}_{s,a,b} \implies R' \in ARS(R)
\end{equation}
where $\mathcal{H}^{R}_{s,a,b} = \big\{ \wb \mid \mathbf{w}^T (\Phi_{\pi^*_{R}}^{(s,a)} - \Phi_{\pi^*_{R}}^{(s,b)}) > 0 \big\}$ and $\mathcal{O} = \{ (s,a,b) | s \in \mathcal{S}, a \in \mathcal{A}_R(s), b \notin \mathcal{A}_R(s) \}$ .
\end{lemma}
\begin{proofsketch}
First we show $\pi_R^*$ is optimal under $R'$ using the policy improvement theorem. Then, using the uniqueness of the optimal value function, we show that all optimal actions under $R$ are also optimal actions under $R'$, and so all optimal policies under $R'$ are optimal under $R$.
    (see Appendix~\ref{proof:rr_vav_thm} for the full proof).
\end{proofsketch}


Lemma~\ref{lem:arp_va} provides a sufficient condition for verifying exact value alignment. We now have the necessary theory to construct an efficient value alignment verification test in the \textit{explicit human, explicit robot} setting. We aim to efficiently verify whether the robot's reward function, $R'$, is within the above intersection of half-spaces, which we call the Aligned Reward Polytope (ARP), as this gives a sufficient condition for $R'$ being value aligned with the human's reward function $R$. Our analysis in this section will be useful later when we consider approximate tests for value alignment verification when one or both of the agents have implicit values.\footnote{Our results may also be of interest in the analysis of \textit{explicit robot, explicit robot} teaming, e.g., ad hoc teamwork~\cite{stone2010ad} where value alignment verification could provide a framework for verifying whether two robots can work together.}


The verification test is constructed by precomputing
the following matrix representation of the ARP:
\begin{equation}\label{eq:arp_delta}
    \mathbf{\Delta} = 
    \begin{bmatrix} \Phi_{\pi^*_{R}}^{(s,a)} - \Phi_{\pi^*_{R}}^{(s,b)} \\
    \vdots
    \end{bmatrix},
\end{equation}
where each row corresponds to a tuple $(s,a,b) \in \mathcal{O}$. Thus, $a$ is an optimal action and $b$ is a suboptimal action under $R$ and each row of $\mathbf{\Delta}$
represents the normal vector for a strict half-space constraint based on feature count differences between an optimal and suboptimal action.
Note that, using this notation, exact value alignment can now be verified by checking whether
$\mathbf{\Delta} \mathbf{w'} > 0 $.
This test can be made more efficient by only including non-redundant half-space normal vectors in $\mathbf{\Delta}$. In Appendix~\ref{app:halfspace_redundancy_removal} we discuss a straightforward linear programming technique to efficiently obtain the minimal set of half-space constraints that define the intersection of half-spaces specified in Lemma~\ref{lem:arp_va}.

\subsubsection{Case 2: Reward Queries}
We now consider the case where the tester can query for samples of the robot's reward function $R'(s)$. Verifying alignment via queries to $R'(s)$ can be reduced to Case (1) by querying the robot for $R'(s)$ over a sufficient number of states and then solving for a system of linear equations to recover $\wb'$, since we assume both the human and robot have access to the reward features $\phi(s)$.\footnote{Note that our results also hold for rewards that are functions of $(s,a)$ and $(s,a,s')$.} Let $\mathbf{\Phi}$ be defined as the matrix where each row corresponds to the feature vector $\phi(s)^T$ for a distinct state $s \in \mathcal{S}$. Then, the number of required queries is equal to $rank(\mathbf{\Phi})$ since we only need samples corresponding to linearly independent rows of $\Phi$. Thus, if $\wb' \in \mathbb{R}^k$, in the worst case we only need $k$ samples from the robot's reward function, since we have $rank(\mathbf{\Phi}) \leq k$. If there is noise in the sampling procedure, then linear regression can be used to efficiently estimate the robot's weight vector $\wb'$. Given $\wb'$ we can verify value alignment by checking whether $\mathbf{\Delta} \mathbf{w'} > 0 $.

\subsubsection{Case 3: Value Function Queries}
Given query access to the robot's state and state-action value functions, $\wb'$ can be determined by noting that $R'(s) = {\wb'}^T \phi(s)$ and
\begin{equation}
    R'(s) = Q^*_{R'}(s,a) - \gamma \mathbb{E}_{s'} \left[V^*_{R'}(s')\right].
\end{equation} 
Computing the expectation requires enumerating successor states. If we define the maximum degree of the MDP transition function as 
\begin{equation}
    d_{\max} = \max_{s \in \mathcal{S},a \in \mathcal{A}} |\{s'\in \mathcal{S} \mid P(s,a,s')>0  \}|,
\end{equation} 
then at most the $d_{\max}$ possible next state value queries are needed to evaluate the expectation. Thus, at most $\mathrm{rank}(\Phi) (d_{\max} + 1)$ queries to the robot's value functions are needed to recover $\wb'$, and the tester can verify value alignment via Case (1). Since $\mathrm{rank}(\Phi) \leq k$ as before, at most $k (d_{\max} + 1)$ queries are required for $\wb' \in \mathbb{R}^k$.

\subsubsection{Case 4: Preference Queries}
Finally, we consider the \textit{implicit robot} setting where the tester can only query the robot for preferences over trajectories, $\xi$. Each preference over trajectories, $\xi_A \prec \xi_B$, induces the constraint ${\wb'}^T (\Phi(\xi_B) - \Phi(\xi_A))>0$, where $\Phi(\xi) = \sum_{i=1}^n \gamma^i \phi(s_i)$ is the cumulative discounted reward features along a trajectory. Thus, our choice set of tests, $\mathcal{T}$, consists of all trajectory preference queries, and we can guarantee value alignment if we have a test $T$ such that $\wb^T(\Phi(\xi_B) - \Phi(\xi_A)) > 0, \forall (\xi_A, \xi_B) \in T$ implies that $\wb \in \bigcap \mathcal{H}^{R}_{s,a,b}$.
We can then construct $\mathbf{\Delta}$ in a similar fashion as above, except each row corresponds to a half-space normal resulting from a preference over individual trajectories (see Appendix~\ref{proof:rr_vav_thm}). Only a logarithmic number of preferences over randomly generated trajectories are needed to accurately represent $\bigcap \mathcal{H}^{R}_{s,a,b}$ via intersection of half-spaces formed by the rows in $\mathbf{\Delta}$ \cite{brown2019drex}.

\subsection{Value Alignment Verification Heuristics}\label{sec:heuristics}
In the next section we relax our assumptions on the robot and consider the \textit{explicit human, implicit robot} setting, where the human seeks to verify value alignment but the robot has a black-box policy that only affords action queries. In this case, we resort to heuristics for value alignment as exact value alignment verification becomes impossible, and $\epsilon$-value alignment verification by directly attempting to solve Equation~\eqref{eq:evav_problem} when $\mathcal{T}$ consists of state-action queries is computationally intractable. As we discuss in detail in Appendix~\ref{app:brute_force_sa_vav}, a direct optimization approach would involve estimating $\Pi$ by computing the optimal policies for a large number of different reward functions, evaluating each policy under $\wb$ to determine which policies are not $\epsilon$-aligned with the tester's reward function $R$, and then solving a combinatorial optimization problem over all possible state queries.

Instead, we resort to efficient heuristics. We consider three heuristic alignment tests designed to work in the black-box value alignment verification setting, where the tester can only ask the robot policy action queries over states.
Each heuristic test consists of a method for selecting states at which to test the robot by querying for an action from the robot's policy and checking if that action is an optimal action under the human's reward function.
Note that querying only a subset of states for robot actions is fundamentally limited to value alignment verification tests with $\delta_{\rm fpr} > 0$ since we will never know for sure that the agent will not take a different action in that state if we query its policy again. Thus, receiving the ``right answer"---an optimal action under the tester's reward $R$---to an action query in a state is not a sufficient condition for exact value alignment.
We briefly discuss three action query heuristics with full details in Appendix~\ref{app:heuristics}. Figure~\ref{fig:island_nav} shows examples of the state queries generated by each heuristic in a simple gridworld.


\paragraph{Critical States Heuristic}
Our first heuristic is inspired by the notion of \textit{critical states}: states where $Q_{R}^*(s,\pi_{R}^*(s)) - \frac{1}{|\mathcal{A}|} \sum_{a \in \mathcal{A}} Q_{R}^*(s,a) > t$, and $t$ is  a user defined threshold~\cite{huang2018establishing}. We adapt this idea to form a critical state alignment heuristic test (CS) consisting of critical states under the human's reward function $R$.
Intuitively, these states are likely to be important; however, often many critical states will be redundant since different states are often important for similar reasons (see Figure~\ref{fig:island_nav}).


\paragraph{Machine Teaching Heuristic}
Our next heuristic is based on Set Cover Optimal Teaching (SCOT)~\cite{brown2019machine}, a machine teaching algorithm that approximates the minimal set of maximally informative state-action trajectories necessary to teach a specific reward function to an IRL agent.
\citet{brown2019machine} prove that the learner will recover a reward function in the intersection of halfspaces that define the CRS (Corollary~\ref{corr:cont_crs}).
We generate informative trajectories using SCOT, and turn them into alignment tests by querying the robot for their action at each state along the trajectories.
SCOT replaces the explicit checking of half-space constraints in Section~\ref{sec:exact_alignment_algos} with implicit half-space constraints that are inferred by querying for robot actions at states along trajectories, thus introducing approximation error and the possibility of false positives. Furthermore, generating a test using SCOT is more computationally intensive than generating a test via the CS heuristic; however, unlike CS, SCOT will seek to avoid redundant queries by reasoning about reward features over a collection of trajectories.

\paragraph{ARP Heuristic}
Our third heuristic takes inspiration from the definition of the ARP to define a black-box alignment heuristic (ARP-bb). ARP-bb first computes $\mathbf{\Delta}$ (see Equation~\eqref{eq:arp_delta}), removes redundant half-space constraints via linear programming, and then only queries for robot actions from the states corresponding to the non-redundant constraints (rows) in $\mathbf{\Delta}$.
Intuitively, states that are queried by ARP-bb are important in the sense that taking different actions reveals important information about the reward function. However, ARP-bb uses single-state action queries to approximate checking each half-space constraint. Thus, ARP-bb trades off smaller query and computational complexity with the potenital for larger approximation error.


\subsection{Implicit Value Alignment Verification}
\label{sec:human}
We now discuss value alignment verification in the \textit{implicit human, implicit robot} setting. Without an explicit representation of the human's values we cannot directly compute the aligned reward polytope (ARP) via enumeration over states and actions to create an intersection of half-spaces as described above.
Instead, we propose the pipeline outlined in Figure~\ref{fig:huma_testing_robot_or_human} where an AI system elicits and distills human preferences and then generates a test which can be used to approximately verify the alignment of any rational agent. 

As is common for active reward learning algorithms~\cite{biyik2019asking}, we assume that the preference elicitation algorithm outputs both a set of preferences over trajectories $\mathcal{P} = \{(\xi_i, \xi_j) : \xi_i \succ \xi_j \}$ and a set of reward weights $\wb$ sampled from the posterior distribution $\{\wb_i\} \sim P(\wb | \mathcal{P})$. Given $\mathcal{P}$ and $P(\wb |\mathcal{P})$, the ARP of the human's implicit reward function can be approximated as  \begin{equation}\label{eq:cont_arp}
ARP(R) \approx \bigcap_{(\xi_i, \xi_j) \in \mathcal{P}} \{ \wb \mid \mathbf{w}^T (\Phi ( \xi_i) - \Phi(\xi_j)) > 0 \big\},
\end{equation}
which generalizes the definition of the ARP to MDPs with continuous states and actions. To see this, note that the intersection of half-spaces in Lemma~\ref{lem:arp_va} enumerates over states and pairs of optimal and suboptimal actions under the human's reward $R$ to create the set of half-space normal vectors $\mathbf{\Delta}$, where each normal vector is a difference of expected feature counts. This enumeration can only be done in discrete MDPs. Equation~\eqref{eq:cont_arp} approximates the ARP for continuous MDPs via half-space normal vectors constructed with empirical feature count differences obtained from pairs of actual trajectories over continuous states and actions.

This test can be further generalized to $\epsilon$-value alignment (Definition~\ref{def:eps_va}) to test agents with bounded rationality or slightly misspecified reward functions.
One method of constructing an $\epsilon$-alignment test is to use the mean posterior reward $\mathbb{E}[\wb]$ to approximate the value difference of each pair of trajectories $\mathbb{E}[\mathbf{\wb}] (\Phi (\xi_i) - \Phi(\xi_j))$, and only include preference queries with estimated value differences of at least $\epsilon$. 
A robot with implicit values is verified as $\epsilon$-value aligned by test $T$ if its preferences over each pair of trajectories in $T$ match the preferences provided by the human 
(see Appendix~\ref{app:human-tester} for more details).

\section{Experiments}\label{sec:results}
We now study the empirical performance of value alignment verification tests, first in the \textit{explicit human} setting and then in the \textit{implicit human} setting.
\subsection{Value Alignment Verification with Explicit Human}
We first study the \textit{explicit human} setting and analyze the efficiency and accuracy of exact value alignment verification tests and heuristics. We consider querying for the weight vector of the robot (ARP-w), querying for trajectory preferences (ARP-pref), and the action-query heuristics: CS, SCOT, and ARP-bb, described in Section~\ref{sec:heuristics}.

\subsubsection{Case Study}\label{sec:case_study}
To illustrate the types of test queries found via value alignment verification, we consider two domains inspired by the AI safety gridworlds~\cite{leike2017ai}. The first domain, \textit{island navigation} is shown in Figure~\ref{fig:island_nav}. Figure~\ref{subfig:island_pi} shows the optimal policy under the tester's reward function,
    $R(s) = 50 \cdot \mathbf{1}_{\rm green}(s) - 1 \cdot \mathbf{1}_{\rm white}(s) - 50 \cdot \mathbf{1}_{\rm blue}(s)$,
where $\mathbf{1}_{\rm color}(s)$ is an indicator feature for the color of the grid cell. Shown in figures~\ref{subfig:island_pref1} and~\ref{subfig:island_pref2} are the two preference queries generated by ARP-pref which consist of pairwise trajectory queries (black is preferable to orange under $R$). Preference query 1 verifies that the robot would rather move the to terminal state (green) rather than visit more white cells. Preference query 2 verifies that the robot would rather visit white cells than blue cells.
Figures~\ref{subfig:island_arpbb},~\ref{subfig:island_scot}, and~\ref{subfig:island_cs} show action query tests designed using the ARP-bb, SCOT, and CS heuristics. The robot is asked which action its policy would take in each of the states marked with a question mark. To pass the test, the agent must respond with an optimal action under the human's policy in each of these states.
ARP-bb chooses two states based on the half-space constraints defined by the expected feature counts of $\pi^*_R$, resulting in an small but myopic test. SCOT queries over a maximally informative trajectory that starts near the water, but includes several redundant states. CS only reasons about Q-value differences and asks many redundant queries (see Appendix~\ref{app:case_study} for more results).

\begin{figure*}[h]
\centering
\begin{subfigure}[b]{0.16\linewidth}
    \centering
    \includegraphics[width=\textwidth]{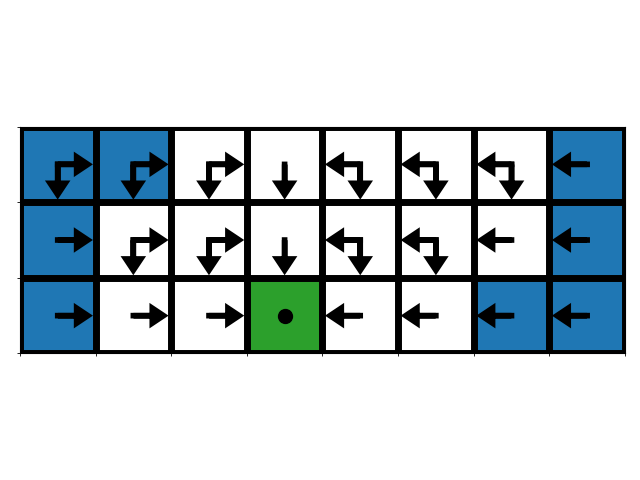}
    \vspace{-1cm}
    \caption{Optimal policy}
    \label{subfig:island_pi}
\end{subfigure}
\begin{subfigure}[b]{0.16\linewidth}
    \centering
    \includegraphics[width=\textwidth]{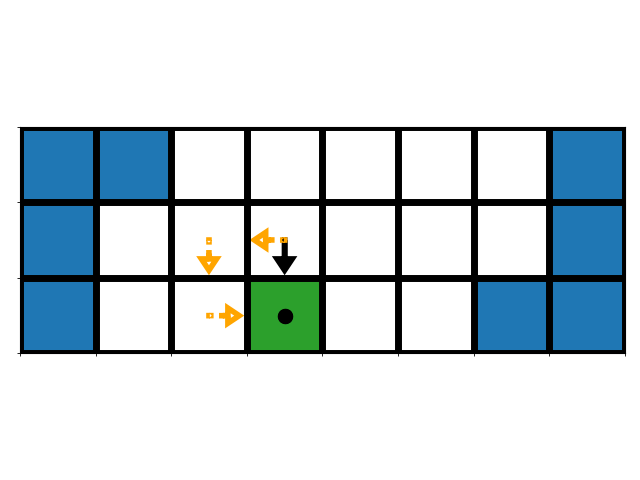}
    \vspace{-1cm}
    \caption{Preference query 1}
    \label{subfig:island_pref1}
\end{subfigure}
\begin{subfigure}[b]{0.16\linewidth}
    \centering
    \includegraphics[width=\textwidth]{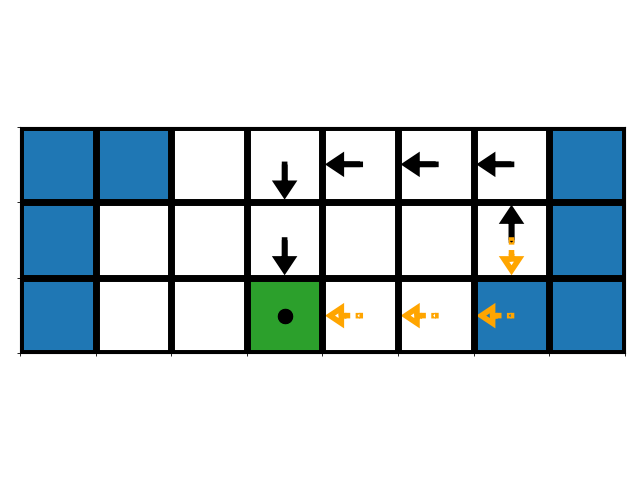}
    \vspace{-1cm}
    \caption{Preference query 2}
    \label{subfig:island_pref2}
\end{subfigure}
\begin{subfigure}[b]{0.16\linewidth}
    \centering
    \includegraphics[width=\textwidth]{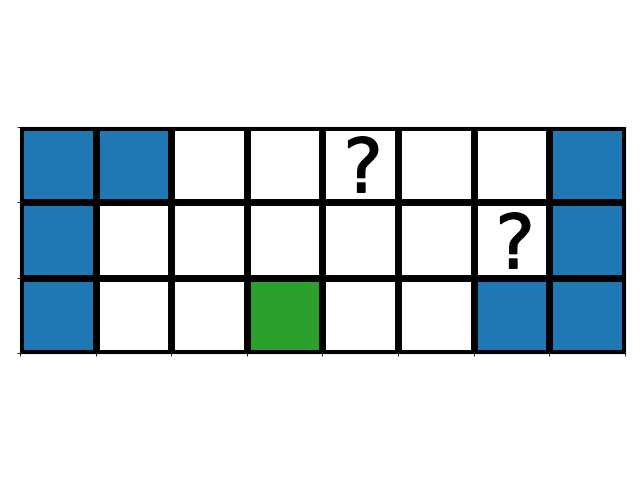}
    \vspace{-1cm}
    \caption{ARP-bb queries}
    \label{subfig:island_arpbb}
\end{subfigure}
\begin{subfigure}[b]{0.16\linewidth}
    \centering
    \includegraphics[width=\textwidth]{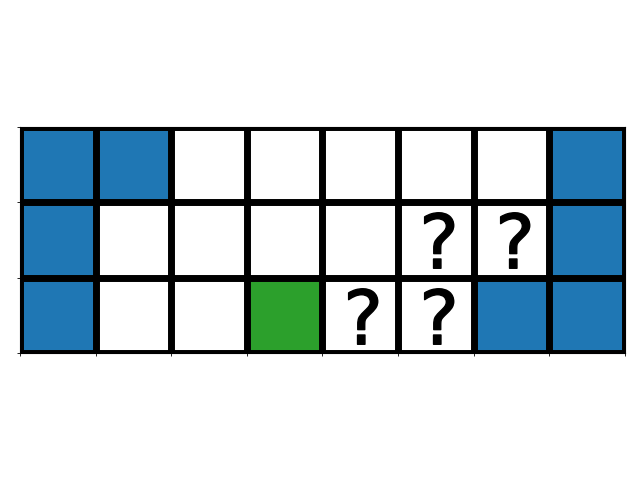}
    \vspace{-1cm}
    \caption{SCOT queries}
    \label{subfig:island_scot}
\end{subfigure}
\begin{subfigure}[b]{0.16\linewidth}
    \centering
    \includegraphics[width=\textwidth]{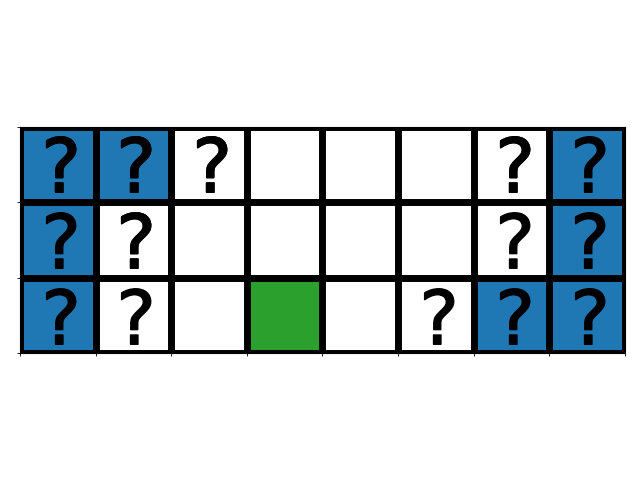}
    \vspace{-1cm}
    \caption{CS queries}
    \label{subfig:island_cs}
\end{subfigure}
\caption{Examples of exact and heuristic value alignment verification tests for an island navigation gridworld~\cite{leike2017ai}. Only two preference queries (b) and (c) are required to provably verify any robot policy (black should be preferred over orange). Figures (d)-(f) show heuristic tests that query for actions at individual states.}
\label{fig:island_nav}
\end{figure*}

\subsubsection{Sensitivity Analysis}

We also analyze the accuracy and efficiency of value alignment verification in the \textit{explicit human, explicit robot} and \textit{explicit human, implicit robot} settings for verifying exact value alignment. We analyze performance across a suite of random grid navigation domains with varying numbers of states and reward features. 
We summarize our results here and refer the reader to Appendix~\ref{app:vav_idealized} for more details. 
As expected, ARP-w and ARP-pref result in perfect accuracy.
SCOT has uses fewer samples than the CS heuristic while achieving nearly perfect accuracy. ARP-bb results in higher accuracy tests, but generates more false positives than SCOT. CS has significantly higher sample cost than the other methods and requires careful tuning of the threshold $t$ to obtain good performance. 
Our results indicate that in the \textit{implicit robot} setting, ARP-pref and  ARP-bb provide highly efficient verification tests. Out of the action query heuristics, SCOT achieved the highest accuracy, while having larger sample complexity than ARP-bb, but achieving lower sample complexity than CS. 




\subsection{Value Alignment Verification with Implicit Human}


We next analyze approximate value alignment verification in the continuous autonomous driving domain from \citet{sadigh2017active}, shown in Figure~\ref{fig:pref}, where we study the \textit{implicit human, implicit robot} setting and consider verifying $\epsilon$-value alignment.
As depicted in Figure~\ref{fig:huma_testing_robot_or_human} we analyze the use of active preference elicitation~\citep{biyik2019asking} to perform value alignment verification with \textit{implicit human values}. We first analyze implicit value alignment verification using preference queries to a synthetic human oracle unobserved ground-truth reward function $R$.

We collected varying numbers of oracle preferences, and computed a non-redundant $\epsilon$-alignment test as described in~\ref{sec:human} and Appendix~\ref{app:halfspace_redundancy_removal}. 
Tests were evaluated for accuracy relative to a set of test reward weights. See Appendix~\ref{app:exp_details} for experimental parameters and details of the testing reward generation protocol.
Figure~\ref{fig:acc} displays the results of the synthetic human experiments. The best tests achieved 100\% accuracy. Although collecting additional synthetic human queries consistently improved verification accuracy, above 50 human queries, accuracy gains were minimal, demonstrating the potential for human-in-the-loop preference elicitation. Furthermore, the generated verification tests were often succinct: one of the tests with perfect accuracy required only six questions out of the original 100 elicited preferences.
Additional experiments and results are detailed in Appendix~\ref{app:exp_details}, including false positive and false negative rate plots, and different methods of estimating the value gap of questions. We also ran an initial pilot study using real human preference labels which resulted in a verification test that achieves 72\% accuracy.


\begin{figure}
\centering
\begin{subfigure}[b]{0.35\linewidth}
\centering
        \includegraphics[width=\textwidth]{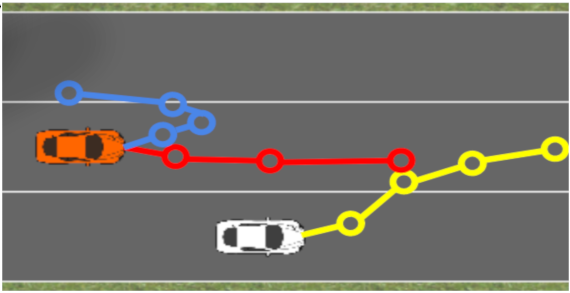}
        \vspace{.1cm}
        \caption{Driving domain 
        }
        \label{fig:pref}
\end{subfigure}
\hfill
\begin{subfigure}[b]{0.64\linewidth}
    \centering
    \includegraphics[width=\textwidth]{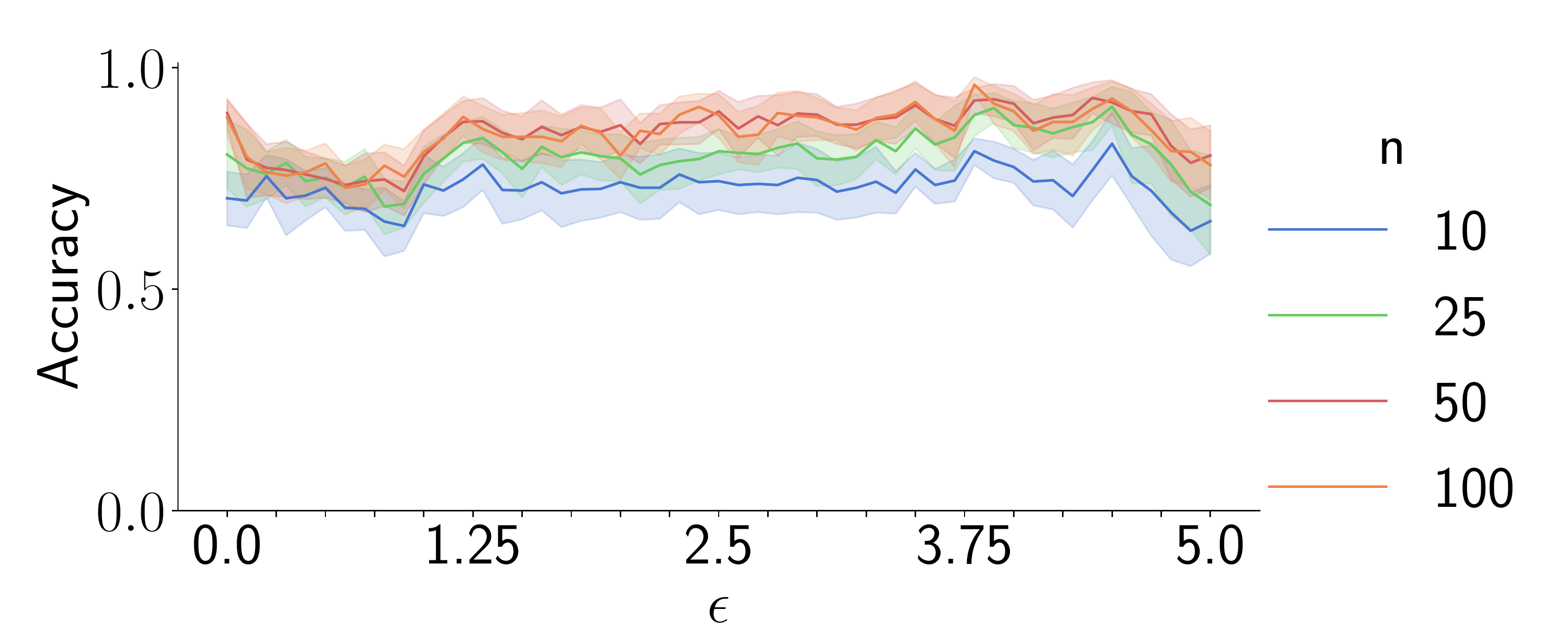}
    \caption{Alignment test accuracy vs $\epsilon$}
    \label{fig:acc}
\end{subfigure}

\caption{\textbf{Implicit human, implicit robot: } $\epsilon$-value alignment verification in a continuous autonomous driving domain. (a) A preference query. The human is asked if they prefer the blue or the red trajectory w.r.t. the trajectory of the white car. (b) $80\%$ confidence intervals on verification accuracy for different values of $\epsilon$, different human query budgets $n$, averaged over ten seeds. 
}
\end{figure}




\section{Generalization to Multiple MDPs} \label{sec:omnipotent_vav}
Up to this point, we have considered designing value alignment tests for a single MDP; however, it is also interesting to try and design value alignment verification tests that enable generalization, e.g., if a robot passes the test, then this verifies value alignment across many different MDPs.

As a step towards this goal, we present a result in the \textit{explicit human, explicit robot} setting where the human can construct testing environments. We consider the idealized setting of an omnipotent tester that is able to construct a set of arbitrary test MDPs and can query directly for the entire optimal policy of the robot in each MDP.
This tester aims to verify value alignment across an infinite family of environments that share the same reward function.
Our result builds on prior analysis on the related problem of omnipotent \textit{active reward learning}. \citet{amin2016towards} prove that an active learner can determine the reward function of another agent within $\epsilon$ precision via $O( \log|\mathcal{S}| + \log(1/\epsilon))$ policy queries.
By contrast, we prove in the following theorem that the sample complexity of $\epsilon$-value alignment verification is only $O(1)$ (see Appendix~\ref{proof:omnipotent_evav} for the proof).

\begin{theorem}\label{thm:omnipotent_vav}
Given a testing reward $R$ (not necessarily linear in known features), there exists a two-query test (complexity $O(1)$) that determines $\epsilon$-value alignment of a rational agent over all MDPs that share the same state space and reward function $R$, but may differ in actions, transitions, discount factors, and initial state distribution.
\end{theorem}

We also note that if the human has access a priori to a finite set of MDPs over which they want to verify value alignment, then our results from earlier sections on exact, heuristic, and approximate value alignment could be extended to this setting. For example, we can define a generalized aligned reward polytope for a family of MDPs as the intersection of the aligned reward polytope for each individual MDP. This intersection of half-spaces provides a sufficient condition for testing value alignment across the entire family of MDPs. 
We leave this as a promising area for future work.

\section{Discussion}\label{sec:generalize}
We analyzed the problem of efficient value alignment verification: how to generate an efficient test that can be used to verify the value alignment of another agent with respect to the human's reward function. We developed a theoretical foundation for value alignment verification and proved sufficient conditions for verifying the alignment of a rational agent under explicit and implicit values for both the human and robot. Our empirical analysis demonstrates that action query heuristics can achieve low sample complexity and high accuracy while only requiring black-box access to an agent's policy. When the human has only implicit access to their values, we analyzed active preference elicitation algorithms as a potential means to automatically construct an approximate value alignment test that can efficiently test another agent with implicit values. 

The biggest assumption we make is that the reward function is a linear combination of features shared by both the human and robot. We would like to emphasize three points:
First, on representing rewards as linear combinations of features, note that the features can be arbitrarily complex, and can even be features learned via a deep neural network which are then linearly transformed by a final linear layer~\cite{brown2020safe}. 
Second, there is the issue of the human and the robot sharing the features. The reason this might actually be a reasonable assumption is that recent techniques enable robots to detect when they cannot explain human input with their existing features and ask for new input specific to the missing features~\cite{bobu2020quantifying,bobu2021feature}, thereby explicitly aligning the robot's reward representation with the human's reward representation. 
Third, even if the features are not perfectly aligned, our approach can still provide value by learning a linear combination of features that approximates the human's reward function to design an alignment test.

Our pilot study with the driving simulation hints that this might be the case, as it gives evidence that value alignment verification is possible when using real human preferences that are determined using pixel-based observations.
Furthermore, the only true requirement for generating value alignment tests that query for robot actions or preferences is for the tester to have a reward function that can be approximated by a linear combination of features.
Thus, these tests could be possibly be applied in cases where a human uses a linear combination of learned or human-designed features to construct an approximate alignment test for robots who have pixel-based policies and/or rewards. 

In conclusion, we believe that value alignment verification is an important problem of practical interest, as it seeks to enable humans to verify and build trust in AI systems. It may also be possible for a robot to use value alignment verification to verify the performance of a human, e.g., AI-generated assessment tests. Future work also includes relaxing rationality assumptions, analyzing value alignment verification tests in more complex domains, and performing a full user study to better analyze the use of human preferences for alignment verification. 

\section*{Acknowledgements}
We would like to thank the ICML anonymous reviewers and Stephen Giguere for their suggestions for improving the paper. 
This has taken place in the Personal Autonomous Robotics Lab (PeARL) at The University of Texas at Austin and the InterACT Lab at the University of Berkeley, California. PeARL research is supported in part by the NSF (IIS-1724157, IIS-1638107, IIS-1749204, IIS-1925082), ONR (N00014-18-2243), AFOSR (FA9550-20-1-0077), and ARO (78372-CS).  This research was also sponsored by the Army Research Office under Cooperative Agreement Number W911NF-19-2-0333. InterACT Lab research is supported in part by AFOSR, NSF NRI SCHOOL, and ONR YIP.
The views and conclusions contained in this document are those of the authors and should not be interpreted as representing the official policies, either expressed or implied, of the Army Research Office or the U.S. Government. The U.S. Government is authorized to reproduce and distribute reprints for Government purposes notwithstanding any copyright notation herein.

\bibliography{driving_test}
\bibliographystyle{icml2021}

\onecolumn
\appendix
\setcounter{theorem}{0}
\setcounter{proposition}{0}
\setcounter{corollary}{0}
\setcounter{lemma}{0}
\setcounter{definition}{1}

\section{Theory and Proofs}


\subsection{Value Alignment Verification of Black-Box Agents}\label{app:bb_impossible}
Definition~\ref{def:eps_va} makes no assumptions about the robot agent except that it acts according to some policy $\pi'$. Given this assumption, how can a tester with reward function $R$ efficiently solve the value alignment verification problem? A brute force attempt at verification would be to query $\pi'(s)$ at every state. However, querying the robot's policy at every state is expensive for discrete problems, and impossible for many real-world problems with continuous state-spaces. If we are able to query for action probabilities at every state, then for discrete MDPs we can verify value alignment by checking whether 
$\{ a \mid \pi'(a|s) > 0 \} \subseteq  \arg \max_a Q_{R^{*}}^*(s,a), \forall s \in \mathcal{S}$.
However, in the case of black-box value alignment verification where the tester only has sample access to the robot's policy without any additional assumptions about policy structure or rationality, we have the following impossibility result:

\begin{proposition}\label{thm:impossibility}
Even in a finite MDP (i.e., $|\mathcal{S}|, |\mathcal{A}| < \infty$), exact value alignment verification via sampling or observing actions from a black-box policy $\pi'$ is impossible in a finite number of queries.
\end{proposition}
\begin{proof}
Consider the robot policy $\pi'$ that takes actions uniformly at random and assume that the tester's reward function is non-trivial, i.e., $\lnot \exists c \in \mathbb{R}, \forall s \in \mathcal{S}, R(s) = c$. Given a finite number of queries, there is a $(\sum_{a \in \arg\max_{a} Q^*_{R}(s,a)} \pi(a|s))^n$ probability that every time the policy is queried at a state $s$ it will select any optimal action $a \in \arg\max_{a} Q^*_{R}(s,a)$. This proves the existence of a non-value aligned policy $\pi'$ that has non-zero probability of being certified as value aligned. In the worst-case, almost surely verifying the value alignment of a policy requires an infinite number of policy queries.
\end{proof}



\subsection{Value Alignment for Rational Agents} \label{app:va_explicit}

We define $OPT(R) = \{\pi \mid \forall \pi \in \Pi, \pi(a|s) > 0 \Rightarrow a \in \arg\max_a Q^*_{R}(s,a) \}$, the set of all optimal (potentially stochastic) policies in MDP $M=(E,R)$ where $\arg \max_x f(x) := \{x \mid f(y) \leq f(x), \forall y \}$ is the set of all function maximizing arguments. We now prove the following:

\begin{corollary}\label{cor:subset_impl_exact}
We have \textbf{exact value alignment} in environment $E$ between a rational robot with reward function $R'$ and a human with reward function $R$ if
$OPT(R') \subseteq OPT(R)$.
\end{corollary}
\begin{proof}

If $OPT(R') \subseteq OPT(R)$ then since $\pi' \in OPT(R')$ we have $\pi' \in OPT(R)$. By construction, if a policy is in $OPT(R)$ then it is optimal under $R$, and so exact alignment immediately follows.

\end{proof}

\subsection{Value Alignment Verification with Explicit Values}\label{proof:rr_vav_thm}

In this section we prove the main theorem of our paper, that efficient exact value alignment verification is possible in many settings. We start with a lemma, equivalent to the case where we have query access to the robot's reward function. We will reduce many of the cases of Theorem~\ref{thm:rational_vav_equivalence} to this case.

\begin{lemma}
Given an MDP $M = (E,R)$, assuming the human's reward function $R$, and the robot's reward function $R'$ can be represented as linear combinations of features $\phi(s) \in \mathbb{R}^k$, i.e., $R(s) = \wb^T \phi(s)$, $R'(s) = {\wb'}^T \phi(s)$, and given an optimal policy $\pi_R^*$ under $R$ then
\begin{equation}
\wb' \in \bigcap_{(s,a,b) \in \mathcal{O}} \mathcal{H}^{R}_{s,a,b} \implies R' \in ARS(R)
\end{equation}
where $\mathcal{H}^{R}_{s,a,b} = \big\{ \wb \mid \mathbf{w}^T (\Phi_{\pi^*_{R}}^{(s,a)} - \Phi_{\pi^*_{R}}^{(s,b)}) > 0 \big\}$ and $\mathcal{O} = \{ (s,a,b) | s \in \mathcal{S}, a \in \mathcal{A}_R(s), b \notin \mathcal{A}_R(s) \}$ .
\end{lemma}
\begin{proof}
We will prove that $\bigcap_{(s,a,b) \in \mathcal{O}} \mathcal{H}^{R}_{s,a,b} \subseteq ARS(R)$. Consider an arbitrary $\wb' \in  \bigcap_{(s,a,b) \in \mathcal{O}} \mathcal{H}^{R}_{s,a,b}$.  
By assumption we have that 

\begin{align}
    \forall s \in \mathcal{S}, \forall a \in \mathcal{A}_R, b \notin \mathcal{A}_R, \wb'^T \Phi_{\pi_R^*}^{(s, a)} &> \wb'^T \Phi_{\pi_R^*}^{(s, b)} \\
    Q_{R'}^{\pi_R^*}(s, a) &> Q_{R'}^{\pi_R^*}(s, b) \label{eq:my_value_your_policy}
\end{align}

Under $R'$ the Q-value of all actions that $\pi_R^*$ does not take are strictly worse than that of the actions that it does take, and so $\pi_R^*$ is optimal under $R'$ by the policy improvement theorem.

Since $\pi_R^*$ is optimal under $R'$, $Q_{R'}^{\pi_R^*}(s, a) = Q_{R'}^*(s, a)$. Thus Eq. \ref{eq:my_value_your_policy} becomes

\begin{align}
    \forall s \in \mathcal{S}, \forall a \in \mathcal{A}_R, b \notin \mathcal{A}_R,
    Q_{R'}^*(s, a) &> Q_{R'}^*(s, b) \label{eq:my_value_opt_policy}
\end{align}

Now consider an arbitrary optimal policy under $R'$, call it $\pi^*_{R'}$. Assume for contradiction that $\pi^*_{R'} \notin OPT(R)$. Therefore, there exists $s\in\mathcal{S}$, $a \in \mathcal{A}_R(s)$, and $b \notin \mathcal{A}_R(s)$ such that 
$Q_{R'}^*(s, b) \geq Q_{R'}^*(s, a)$. However, this contradicts \eqref{eq:my_value_opt_policy}. Thus, we have $\pi^*_{R'} \in OPT(R)$ and since $\pi^*_{R'}$ was assumed to be any optimal policy under $R'$, we have
$\wb' \in \bigcap_{(s,a,b) \in \mathcal{O}} \mathcal{H}^{R}_{s,a,b}$ implies $OPT(R') \subseteq OPT(R)$ and so $\bigcap_{(s,a,b) \in \mathcal{O}} \mathcal{H}^{R}_{s,a,b} \subseteq ARS(R)$ by Definition~\ref{def:ARP}.
\end{proof}

We now prove the full theorem.

\begin{theorem}
Under the assumption of a rational robot that shares linear reward features with the human, efficient exact  value alignment verification is possible in the following query settings: (1) Query access to reward function weights $\mathbf{w}'$, (2) Query access to samples of the reward function $R'(s)$, (3) Query access to $V^*_{R'}(s)$ and $Q^*_{R'}(s,a)$, and (4) Query access to preferences over trajectories.
\end{theorem}\begin{proof}
The proof of case (1) follows directly from Lemma~\ref{lem:arp_va}.

In case (2), the tester can query for samples of the reward function $R'(s)$. If the tester only has query access to $R'(s)$, then the weight vector $\wb'$ can be recovered by solving a linear system.

\begin{equation}
    \mathbf{R}'_{\mathrm{sample}} = \begin{bmatrix}
    \vdots \\
    R(s_i) \\
    \vdots
    \end{bmatrix} = \Phi_{\mathrm{sample}} \wb' = \begin{bmatrix}
    \vdots \\
    \Phi(s_i) \\
    \vdots
    \end{bmatrix} \wb'
\end{equation}

This system is guaranteed to have a unique solution if $\mathrm{rank}(\Phi_{\mathrm{sample}}) = k$ i.e. $\Phi_{\mathrm{sample}}$ is full column rank. If $\Phi$, the matrix of features at every state, is full column rank, then there is a subset of $k$ rows which is also full column rank. If $\Phi$ is not full column rank, then there is some feature column $\phi_i$ which is a linear combination of other feature columns, and so can be removed from the test without affecting the predicted alignment of any policy. Features can be safely removed in this manner until the remaining columns are linearly independent. Thus for any environment there is a set of $k$ states which one can query $R'(s)$ at in order to recover a sufficient subset of the reward weights for value alignment purposes. Note that this also works for rewards that are functions of $(s,a)$ and $(s,a,s')$.

If $R'(s)$ is a stochastic function, then linear regression can be used to efficiently estimate the robot's weight vector $\wb'$. After recovering the weight vector, the same value alignment test used for case (1) can be used.

In case (3) the tester has access to the value functions of the robot. If the tester can query the robot agent's value function then $R(s)$ can be recovered from the Bellman equation
\begin{equation}
    R(s) = Q_{R}^*(s,a) - \gamma \mathbb{E}_{s'|s,a} \left[V_{R}^*(s')\right]
\end{equation}

Computing the expectation requires enumerating successor states. If we define the maximum degree of the MDP transition function as 
\begin{equation}
    d_{\max} = \max_{s \in \mathcal{S},a \in \mathcal{A}} |\{s'\in \mathcal{S} \mid P(s,a,s')>0  \}|,
\end{equation} 
then at most the $d_{\max}$ possible next state value queries are needed to evaluate the expectation. Thus, at most $\mathrm{rank}(\Phi) (d_{\max} + 1)$ queries to the robot's value functions are needed to recover $\wb'$, and the tester can verify value alignment via Case (1). Since $\mathrm{rank}(\Phi) \leq k$ as before, at most $k (d_{\max} + 1)$ queries are required.

In case (4), the tester only has access to the robot's values via preference queries over trajectories. If the robot agent being tested can answer pairwise preferences over trajectories, then a value alignment test can also be tested via an approximation of the ARP. Each preference over trajectories $\xi_A \prec \xi_B$ induces the constraint $\wb^T (\xi_B - \xi_A)>0$. Thus, given a test $\mathcal{T}$ consisting of preferences over trajectories, we can guarantee value alignment if 
\begin{equation}
    \{ \wb \mid \wb^T(\xi_B - \xi_A) > 0, \forall (\xi_A, \xi_B) \in \mathcal{T} \} \subseteq \text{ARS}(\wb).
\end{equation}
Note that a single trajectory in general will not actually match the successor features of a stochastic policy. However, by synthesizing arbitrary trajectories we can create more half-space constraints than are used to define the ARP since these trajectories do not need to be the product of a rational policy. 
As more trajectory queries are asked the estimate of the ARP will approach a subset of the true ARP. Brown et al. \cite{brown2019drex} proved that given random halfplane constraints, the volume of the polytope will decrease exponentially. Thus we will need a logarithmic number of queries to accurately approximate the ARP.

\end{proof}

\subsection{Relationship of the ARP to Ng and Russell's Consistent Reward Sets} \label{app:ng_russel}
In this section we discuss the relationship between our approach and the foundational work on IRL by Ng and Russell \cite{ng2000algorithms}.

We define the set of rewards consistent with an optimal policy as follows:
\begin{definition}
Given an environment $E$, the \textit{consistent reward set} (CRS) of a policy $\pi$ in environment $E$ is defined as the set of reward functions under which $\pi$ is optimal:
\begin{equation}
\text{CRS}(\pi) =  
\{\mathbf{w} \in \mathbb{R}^k \mid \pi \text{ is optimal with respect to } R(s) = \mathbf{w}^T \phi(s) \}. \end{equation}
\end{definition}

The fundamental theorem of inverse reinforcement learning \cite{ng2000algorithms}, defines the set of all consistent reward functions as a set of linear inequalities for finite MDPs. 



\begin{proposition}\label{thm:ngRussell} \cite{ng2000algorithms}
Given an environment $E$, with finite state and action spaces, $\mathbf{R} \in \text{CRS}(\pi)$ if and only if
\begin{equation}
(\mathbf{P_\pi} - \mathbf{P_a})(\mathbf{I} - \gamma \mathbf{P_\pi})^{-1} \mathbf{R} \geq 0, \; \forall a \in \mathcal{A}
\end{equation}
where $\mathbf{P_a}$ is the transition matrix associated with always taking action $a$, $\mathbf{P_{\pi}}$ is the transition matrix associated with policy $\pi$, and $\mathbf{R}$ is the column vector of rewards for each state in the MDP.
\end{proposition}


When the reward function is a linear combination of features, we get the following:
\begin{corollary} \label{app-corr:cont_crs} \cite{ng2000algorithms,brown2019machine}
Given an environment $E$, the $CRS(\pi)$ is given by the following intersection of half-spaces:
\begin{eqnarray}
\{\mathbf{w} \in \mathbb{R}^k \mid \mathbf{w}^T (\Phi_{\pi}^{(s,a)} - \Phi_{\pi}^{(s,b)}) \geq 0,
\forall a \in \support(\pi(s)), b \in \mathcal{A}, s \in \mathcal{S} \}.
\end{eqnarray}
\end{corollary}
\begin{proof}
In every state $s$ there is one or more optimal actions $a$. For each optimal action $a \in \support(\pi(s))$, we then have by definition of optimality that
\begin{equation}
Q^*(s,a) \geq Q^*(s,b), \; \forall b \in A
\end{equation}
Rewriting this in terms of expected discounted feature counts we have
\begin{equation}
\wb^T \Phi_{\pi}^{(s,a)} \geq \wb^T \Phi_{\pi}^{(s,b)}, \; \forall b \in A
\end{equation}
Thus, the entire feasible region is the intersection of the following half-spaces 
\begin{eqnarray}
\wb^T (\Phi_{\pi}^{(s,a)} -  \Phi_{\pi}^{(s,b)}) \geq 0, \\ \forall a \in \text{support}(\pi(s)), b \in \mathcal{A}, s \in \mathcal{S}
\end{eqnarray}
and thus the feasible region is convex.
\end{proof}

The consistent reward set of a demonstration from an optimal policy can be defined similarly:
\begin{corollary}\label{cor:feasibleDemo} \cite{brown2019machine}
Given a set of demonstrations $\mathcal{D}$ from a policy $\pi$, $CRS(\mathcal{D} | \pi)$ is given by the following intersection of half-spaces:
\begin{equation}
\mathbf{w}^T (\Phi_{\pi}^{(s,a)} - \Phi_{\pi}^{(s,b)}) \geq 0,\; \forall (s,a) \in \mathcal{D}, b \in \mathcal{A}.  
\end{equation}
\end{corollary}

\begin{proof}
The proof follows from the proof of ~\ref{app-corr:cont_crs} by only considering half-spaces corresponding to optimal $(s,a)$ pairs in the demonstration.
\end{proof}


Note that Corollary~\ref{app-corr:cont_crs} does not solve the alignment verification problem. It only provides a necessary, but not sufficient condition. If a reward function is within the CRS of a policy dot not imply all agents optimal under that reward function are aligned. Consider the example of the all zero reward: it is always in the CRS of any policy; however, an agent optimizing the zero reward can end up with any policy. Even ignoring the all zero reward we can have rewards on the boundaries of the CRS polytope that are consistent with a policy, but not value aligned since they lead to more than one optimal policy, one or more of which may not be optimal under the tester's reward function. 

\subsection{Proof of Theorem~\ref{thm:omnipotent_vav}: $\epsilon$-Alignment Verification via Omnipotent Testing}\label{proof:omnipotent_evav}


In this section we consider what is possible in the omnipotent tester case where a tester can design a set of test MDPs in order to verify alignment over a (potentially infinite) family of MDPs that share reward information. We are able to prove that, under some assumptions, alignment over a family of MDPs is possible by querying a complete policy in only two test MDPs.

More formally, we consider the case where the testing agent is able to construct a set of arbitrary test MDPs to verify value alignment across a family of environments that may have different transitions, actions, initial state distribution, and discount factor, but that share the same reward function over states. Amin and Sing \cite{amin2016towards} prove that an omnipotent active learner can determine the reward function of another agent within $\epsilon$ precision via $O( \log(|\mathcal{S}|) + \log(1/\epsilon))$ active policy queries. We extend this result to the case of $\epsilon$-value alignment verification.

Before we prove Theorem~\eqref{thm:omnipotent_vav}, we require the following Lemma, which proves that if two agents' reward functions are similar enough (in an $L^{\infty}$ sense) then we can guarantee $\epsilon$-value alignment.


\begin{lemma} \label{lem:bounded_r_e_vav}
If $\|R(s) - R'(s)\|_\infty \leq \epsilon (1- \gamma)/2$, where $\gamma$ is the discount factor and $\epsilon$ is any non-negative error term, then rational agents that have reward functions $R(s)$ and $R'(s)$ are $\epsilon$-value aligned across all MDPs that share the reward function $R(s)$.
\end{lemma}
\begin{proof}
For $\pi' \in OPT(R')$ to be $\epsilon$-value aligned under the humans' reward function $R$ we must have $\forall s \in \mathcal{S}, V^*_R(s) - V^{\pi'}_{R}(s) \leq \epsilon$. To prove the lemma we must show that an adversary that picks $R'$ within the constraint $\|R(s) - R'(s)\|_\infty \leq \epsilon (1- \gamma)/2$ cannot violate the alignment condition in any MDP. 

The adversary wants to maximize $V^*_R(s)-V^{\pi'}_{R}(s)$ at some state. Let $\pi^*_R \in OPT(R)$ be an optimal policy under $R$. Since $\pi'$ is optimal under $R'$, we have $V_{R'}^{\pi^*_R}(s) - V_{R'}^{\pi'}(s) \leq 0$. We will show bounds on the maximum gap between $V_{R'}$ and $V_{R}$ for both policies, and use those bounds in combination with the above inequality to show that $V_{R}^*(s) - V_{R}^{\pi'}(s) \leq \epsilon$. The adversary would like the resulting upper bound to be as large as possible, which is achieved by making $V_{R}^{\pi_R^*}$ as large as possible and $V_{R}^{\pi'}$ as small as possible, which is in turn achieved by making $V_{R'}^{\pi_R^*}$ as small as possible in relationship to $V_{R}^{\pi_R^*}$ and vice versa for $V_{R'}^{\pi'}$.
Thus the adversary creates $R'$ by subtracting the maximum $\epsilon(1-\gamma)/2$ from the true reward ($R'(s) = R(s) - \epsilon(1-\gamma)/2$) at states visited by $\pi^*_R$ to make them look as bad as possible and adding $\epsilon(1-\gamma)/2$ to the true reward ($R'(s) = R(s) + \epsilon(1-\gamma)/2$) at states visited by $\pi'$ look as good as possible. If $\pi_R^*$ and $\pi'$ visit some of the same states, this assignment of $R'(s)$ isn't possible, but this only tightens our bound. Thus, we have in the worst-case
\begin{align}
    V^{\pi^*_R}_{R'} &= \mathbb{E}[\sum_{t=0}^\infty \gamma^t R'(s_t) \mid s_t \sim \pi^*_R] \\
    &\geq \mathbb{E}[\sum_{t=0}^\infty \gamma^t \left(R(s_t) - \epsilon  (1-\gamma)/2 \right) \mid s_t \sim \pi^*_R]\\
    &\geq V^{\pi^*_R}_R - \frac{\epsilon(1-\gamma)}{2(1-\gamma)} \\
    &\geq V^{\pi^*_R}_R - \frac{\epsilon}{2}
\end{align}

and

\begin{align}
    V^{\pi'}_{R'} &= \mathbb{E}[\sum_{t=0}^\infty \gamma^t R'(s_t) \mid s_t \sim \pi'] \\
    &\leq \mathbb{E}[\sum_{t=0}^\infty \gamma^t \left(R(s_t) + \epsilon(1-\gamma)/2\right) \mid s_t \sim \pi'] \\
    &\leq V^{\pi'}_R + \frac{\epsilon(1-\gamma)}{2(1-\gamma)}\\
    &\leq V^{\pi'}_R + \frac{\epsilon}{2}
\end{align}

As noted above we have $V^{\pi^*_R}_{R'}(s) \leq V^{\pi'}_{R'}(s)$ since $\pi'$ is optimal under $R'$. Substituting the above bounds provides that

\begin{align}
    V^{\pi^*_R}_{R'}(s) &\leq V^{\pi'}_{R'}(s) \\
    V^{\pi^*_R}_{R}(s) - \epsilon/2 &\leq V^{\pi'}_{R}(s) + \epsilon/2 \\
    V^{\pi^*_R}_{R}(s) - V^{\pi'}_{R}(s) &\leq \epsilon
\end{align}
Thus, we have shown that under the assumption that $\|R(s) - R'(s)\|_\infty \leq \epsilon (1- \gamma)/2$, then the robot agent with reward function $R'$ is $\epsilon$-value aligned with the tester's reward function $R$ under all possible MDPs that share the reward function $R$.
\end{proof}

Note that if we scale the reward of an agent by a positive constant or by a constant vector, we can get the difference to look arbitrarily large even if the two rewards lead to the same optimal policy. This is undesirable for computing value alignment in terms of reward differences. Comparing rewards in this way works best if they are similarly normalized. We utilize a canonical form for reward functions defined by the transformation $(R(s) - \max_s R(s)) / (\max_s R(s) - \min_s R(s))$ such that the values of the reward function are scaled to be between 0 and 1 \cite{amin2016towards}. Following the notation of Amin and Singh \cite{amin2016towards} we use $[R]$ to denote 
the canonical form for reward function $R$. Note that we will not assume access to the canonical form of the robot's reward function. Indeed we assume no direct access to this reward function.

Given the ability to construct arbitrary testing environments, we can guarantee $\epsilon$-value alignment over all MDPs that share the reward function $R$.
The following theorem is inspired by Amin and Singh \cite{amin2016towards} who prove an analogous theorem for the case of actively querying an expert's policy to approximate the expert's reward function. The proof of Amin and Singh \cite{amin2016towards} relies on binary search and the query algorithm they derive results in query complexity of $O( \log(|\mathcal{S}|) + \log(1/\epsilon))$, where each query requires the expert to specify a complete policy for a new MDP. In contrast, our proof is based instead on machine testing, and we prove that in the case of value alignment verification we only require $O(1)$ policy queries. In fact we only need two test MDPs with policy queries.

\begin{theorem}
Given a testing reward $R$, there exists a two-query test (complexity $O(1)$) that determines $\epsilon$-value alignment of a rational agent over all MDPs that share the same state space and reward function $R$, but may differ in actions, transitions, discount factors, and initial state distribution.
\end{theorem}
\begin{proof}

By Lemma~\ref{lem:bounded_r_e_vav} we want a test that guarantees $\|[R'] - [R]\|_\infty \leq \epsilon (1-\gamma)/2$. 
Thus we need to show that
\begin{align}
 |[R'](s) - [R](s)| \leq \epsilon (1-\gamma)/2, \forall s \in S
\end{align}
which implies that 
\begin{align}
[R](s) - \epsilon (1-\gamma)/2 \leq [R'](s) \leq [R](s) + \epsilon (1-\gamma)/2, \forall s \in S.
\end{align}
We use the notation $[R]$ and $[R']$ to represent the canonical versions of $R$ and $R'$, the tester's and robot's reward functions, respectively. If we can directly query for $R'$, then we simply compute $\|R-R'\|_\infty$ and check if it is less than $\epsilon (1 - \gamma)/2$. We now consider the case where we can only query the robot's policy. We define $s_{\max} = \arg \max_s R(s)$ and $s_{\min} = \arg \min_s R(s)$ and $s'_{\max} = \arg \max_s R'(s)$ and $s'_{\min} = \arg \min_s R'(s)$.

We first cover the simple case where we only have two states: $s_{\min}$ and $s_{\max}$. In this case, we can construct an MDP with two actions: $a_1$ that always leads to $s_{\min}$ and $a_2$ which always leads to $s_{\max}$. We then can verify value alignment verification by asking for the robot's optimal policy and checking that $a_2$ is always preferred over $a_1$. Note that if the robot has more than two actions, we can simply make all remaining actions equivalent to either $a_1$ or $a_2$ since the tester has full control over the transition dynamics.

We now consider the general case where there are more than two states. We create two testing environments such that from each state there is an action $a_1$ that self transitions and an action $a_2$ that goes from each state to $s_{\max}$ with probability $\alpha_s$ and to $s_{\min}$ with probability $(1-\alpha_s)$, except in states $s_{\min}$ and $s_{\max}$ in which all transitions via $a_1$ and $a_2$ are self transitions. Thus, taking action $a_2$ represents a gamble between the states with minimum and maximum reward under the tester's reward function $R$. 
For $s \in S\setminus\{s_{\max}, s_{\min}\}$, we design two different transition dynamics with the parameters $\alpha^U$ and $\alpha^L$ such that 
 $\alpha^L_s = \max([R](s) - \frac{\epsilon (1-\gamma)}{2},0)$ and $\alpha_s^U = \min([R](s) + \frac{\epsilon (1-\gamma)}{2}, 1)$.
Then we construct two test environments $E_L$ and $E_U$. $L$ has $\alpha^L$ as the transitions and $U$ has $\alpha^U$ as the transitions. We then query the robot for its optimal policy in both test environments and use the policy to answer the two test questions:
\begin{enumerate}
\item Is $\pi(s)=a_1$, $\forall s \in S\setminus \{s_{\min}, s_{\max}\}$ in MDP $L$?
\item Is $\pi(s) = a_2$, $\forall s \in S\setminus \{s_{\min}, s_{\max}\}$ in MDP $U$? 
\end{enumerate}

If the agent answers "YES" to the first question, then $\forall s \in S\setminus\{s_{\max}, s_{\min}\}$ we know that $a_1$ is at lest as good as $a_2$. Thus the agent prefers to self transition at a state rather than take action $a_2$ which leads to a stochastic transitions to either $s_{\max}$ or $s_{\min}$. Thus, under the robot's unknown reward $R'$ the following inequality holds for all $s \in S\setminus\{s_{\max}, s_{\min}\}$:
\begin{align}
&\alpha_s^L R'(s_{\max}) + (1-\alpha_s^L) R'(s_{\min}) \leq R'(s) \\
\Leftrightarrow \quad& \alpha_s^L R'(s_{\max}) + (1-\alpha_s^L) R'(s_{\min}) - R'(s'_{\min}) \leq R'(s) - R'(s'_{\min})\\
\Leftrightarrow \quad& \alpha_s^L (R'(s_{\max}) - R'(s'_{\min}))  + (1-\alpha_s^L) (R'(s_{\min}) - R'(s'_{\min})) \leq R'(s) - R'(s'_{\min})\\
\Leftrightarrow \quad& \alpha_s^L \frac{R'(s_{\max}) - R'(s'_{\min})}{R'(s'_{\max}) - R'(s'_{\min})}  + (1-\alpha_s^L) \frac{R'(s_{\min}) - R'(s'_{\min})}{R'(s'_{\max}) - R'(s'_{\min})} \leq \frac{R'(s) - R'(s'_{\min})}{R'(s'_{\max}) - R'(s'_{\min})}\\
\Leftrightarrow \quad& \alpha_s^L [R'](s_{\max}) + (1-\alpha_s^L) [R'](s_{\min}) \leq [R'](s). \label{eq:lower_bound_omni}
\end{align}
and similarly, if the agent answers "YES" to question 2, we have
\begin{align}
 & R'(s) \leq \alpha_s^U R'(s_{\max}) + (1-\alpha_s^U) R'(s_{\min}) \\
 \Leftrightarrow \quad & [R'](s) \leq \alpha_s^U [R'](s_{\max}) + (1-\alpha_s^U) [R'](s_{\min}).  \label{eq:upper_bound_omni}
\end{align}
These above inequalities hold for all $s \in \mathcal{S}\setminus \{s_{\max}, s_{\min} \}$.

We now prove that answering "YES" to both questions 1 and 2 also means that  $s'_{\max} \equiv \max_s R'(s) = \max_s R(s) \equiv s_{\max}$ and $s'_{\min} \equiv \min_s R'(s) = \min_s R(s) \equiv s_{\min}$. We assume that $\frac{\epsilon (1-\gamma)}{2} < 0.5$ and thus consider three cases for the values of  $\alpha^L_s = \max([R](s) - \frac{\epsilon (1-\gamma)}{2},0)$ and $\alpha_s^U = \min([R](s) + \frac{\epsilon (1-\gamma)}{2}, 1)$
\begin{enumerate}
    \item $\alpha^L_s = 0$ and $\alpha_s^U = [R](s) + \frac{\epsilon (1-\gamma)}{2}$
    \item $\alpha^L_s = [R](s) - \frac{\epsilon (1-\gamma)}{2}$ and $\alpha_s^U = [R](s) + \frac{\epsilon (1-\gamma)}{2}$
    \item $\alpha^L_s = [R](s) - \frac{\epsilon (1-\gamma)}{2}$ and $\alpha_s^U = 1$
\end{enumerate}
\textbf{Case 1:} 
We have $\alpha^L_s = 0$, thus. If the robot answers YES to question 1, we have
\begin{align}
&\alpha_s^L [R'](s_{\max}) + (1-\alpha_s^L) [R'](s_{\min}) \leq [R'](s) \\
\Rightarrow \quad& [R'](s_{\min}) \leq [R'](s) \label{eq:case_1_a}
\end{align}
We also have
\begin{align}
    & [R'](s) \leq \alpha_s^U [R'](s_{\max}) + (1-\alpha_s^U) [R'](s_{\min}). 
\end{align}
plugging in the value in Equation~\eqref{eq:case_1_a} we have
\begin{align}
    & [R'](s) \leq \alpha_s^U [R'](s_{\max}) + (1-\alpha_s^U) [R'](s) \\
    \Rightarrow \quad & [R'](s) - (1-\alpha_s^U) [R'](s) \leq \alpha_s^U [R'](s_{\max}) \\
    \Rightarrow \quad & [R'](s) \leq [R'](s_{\max}) 
\end{align}

\textbf{Case 2:}
We have $\alpha^L_s = [R](s) - \frac{\epsilon (1-\gamma)}{2}$ and $\alpha_s^U = [R](s) + \frac{\epsilon (1-\gamma)}{2}$. Plugging these into Equation~\eqref{eq:lower_bound_omni} we have
\begin{align}
&\alpha_s^L [R'](s_{\max}) + (1-\alpha_s^L) [R'](s_{\min}) \leq [R'](s) \label{eq:case_2_lower}\\
\Rightarrow \quad&   [R'](s_{\min}) \leq \frac{1}{(1-\alpha_s^L)} \left([R'](s) - \alpha_s^L [R'](s_{\max})\right)
\end{align}
Plugging this into the following equation, yields:
\begin{align}
    & [R'](s) \leq \alpha_s^U [R'](s_{\max}) + (1-\alpha_s^U) [R'](s_{\min}) \\
    \Rightarrow \quad & [R'](s) \leq \alpha_s^U [R'](s_{\max}) + (1-\alpha_s^U) \left(\frac{1}{(1-\alpha_s^L)} \left([R'](s) - \alpha_s^L [R'](s_{\max}) \right)\right) \\
    \Rightarrow \quad & (1-\alpha_s^L) [R'](s) \leq (1-\alpha_s^L)\alpha_s^U [R'](s_{\max}) + (1-\alpha_s^U) \left([R'](s) - \alpha_s^L [R'](s_{\max}) \right)
\end{align}
Plugging the values for $\alpha_s^L$ and $\alpha_s^U$ for Case 2 and reducing the resulting algebraic equation results in
\begin{align}
    [R'](s) \leq [R'](s_{\max}) 
\end{align}
We then plug this value into Equation~\eqref{eq:case_2_lower} we get
\begin{align}
    & \alpha_s^L [R'](s_{\max}) + (1-\alpha_s^L) [R'](s_{\min}) \leq [R'](s) \\
    \Rightarrow \quad& \alpha_s^L [R'](s) + (1-\alpha_s^L) [R'](s_{\min}) \leq [R'](s)  \\
    \Rightarrow \quad&  [R'](s_{\min}) \leq [R'](s) .
\end{align}

\textbf{Case 3:}
We have $\alpha^L_s = [R](s) - \frac{\epsilon (1-\gamma)}{2}$ and $\alpha_s^U = 1$. Thus, 
\begin{align}
    & [R'](s) \leq \alpha_s^U [R'](s_{\max}) + (1-\alpha_s^U) [R'](s_{\min}) \\
    \Rightarrow \quad & [R'](s) \leq [R'](s_{\max}) . 
\end{align}
Plugging this into the following equation yields:
\begin{align}
    &\alpha_s^L [R'](s_{\max}) + (1-\alpha_s^L) [R'](s_{\min}) \leq [R'](s) \\
    \Rightarrow \quad& \alpha_s^L [R'](s) + (1-\alpha_s^L) [R'](s_{\min}) \leq [R'](s) \\
    \Rightarrow \quad& [R'](s_{\min}) \leq [R'](s) \\
\end{align}

Thus, for every state $s \in \mathcal{S}\setminus \{s_{\max}, s_{\min} \}$, we have proved that we always have
\begin{align}
    [R'](s_{\min}) \leq [R'](s) \leq [R'](s_{\max}).
\end{align}
Therefore, it must be the case that  $s'_{\max} \equiv \max_s R'(s) = \max_s R(s) \equiv s_{\max}$ and $s'_{\min} \equiv \min_s R'(s) = \min_s R(s) \equiv s_{\min}$.

Combining the above results we have (assuming the robot answers "YES" to questions 1 and 2) that $[R](s_{\max}) = [R'](s_{\max}) = 1$ and $[R](s_{\min}) = [R'](s_{\min}) = 0$. Additionally, for the remaining states, $s \in \mathcal{S}\setminus \{s_{\max}, s_{\min} \}$, we have that 
\begin{eqnarray}
&&\alpha_s^L R'(s_{\max}) + (1-\alpha_s^L) R'(s_{\min}) \leq R'(s) \leq \alpha_s^U R'(s_{\max}) + (1-\alpha_s^U) R'(s_{\min}) \\
 &\Rightarrow&\alpha_s^L  [R'](s_{\max}) + (1-\alpha_s^L) [R'](s_{\min}) \leq [R'](s) \leq \alpha_s^U [R'](s_{\max}) + (1-\alpha_s^U) [R'](s_{\min}) \nonumber \\
  &\Rightarrow &\alpha_s^L  \leq [R'](s) \leq \alpha_s^U  \label{line:0_1_R}\\
    &\Rightarrow& \max([R](s) - \epsilon (1-\gamma)/2,0) \leq [R'](s) \leq \min([R](s) + \epsilon (1-\gamma)/2, 1)\\
    &\Rightarrow& |[R'](s) - [R](s)| \leq \epsilon (1-\gamma)/2.
    \end{eqnarray}
Thus, we have$ \|[R'] - [R] \|_\infty \leq \epsilon (1-\gamma)/2$ so by Lemma~\ref{lem:bounded_r_e_vav} we have verified $\epsilon$-value alignment via two policy preference queries as desired.
\end{proof}

\section{Value Alignment Verification for Action Queries}\label{app:brute_force_sa_vav}

In this section we discuss the difficulty of solving Equation~\eqref{eq:evav_problem} directly to find which states to query for actions. The approach detailed here will generally be intractable, but motivates the tractable heuristics discussed in Section \ref{sec:heuristics}.

We consider the problem of finding a subset of states where we will query the robot for an action they would take at that state. We want to optimize the following objective (copied from the main text for convenience):
\begin{align} \label{app-eq:evav_problem}
    &\min_{T \subseteq \mathcal{T}} |T|, \text{s.t.}\; \forall \pi' \in \Pi, \\
    &V^{*}_{R}(s) - V^{\pi'}_{R}(s)  > \epsilon \Rightarrow Pr[\text{$\pi'$ passes test $T$}] \leq \delta_{\rm fpr} \nonumber \\
    &V^{*}_{R}(s) - V^{\pi'}_{R}(s)  \leq \epsilon \Rightarrow Pr[\text{$\pi'$ fails test $T$}] \leq \delta_{\rm fnr} \nonumber
\end{align}
where the choice set $T \subset \mathcal{T}$ is the set of states where we query for actions from the robot's policy. We seek to use these actions to verify value alignment. Furthermore, we want to precompute a single test that will certify any agent.

We will discuss a naive approach that motivates our heuristics from Section \ref{sec:heuristics}.
We propose a breadth-first search to find the optimal set of test states for action queries.

First we need to establish how likely detecting $\epsilon$-misalignment from a single action query at each state is. Consider all the reward functions that have a policy that is rational under that reward function but is $\epsilon$-value misaligned under $R$: $\mathcal{R}' = \{R' | \exists \pi', \pi' \in OPT(R'), \pi' \mathrm{\ is\ } \epsilon-\mathrm{misaligned under} R\}$. For each policy optimal under a reward in $\mathcal{R}'$, rollout that policy at every state $N$ times. From this we can obtain a Monte-Carlo estimate of the probability of detecting the robot is $\epsilon$-misaligned by taking the ratio of rollouts where the robot takes an $\epsilon$-misaligned action $b$ such that 
\begin{equation}\label{eq:action_query_test}
    Q^*_{R}(s,\pi^*_{R}(s)) - Q^*_{R}(s,b) > \epsilon,
\end{equation}
to the number of rollouts $N$. 

We now perform breadth-first search to search to solve the combinatorial optimization problem of determining the subset of states that allow high-confidence value alignment verification. We use breadth-first search since we are interested in finding the minimal number of states to test such that we can detect all non-aligned agents with probability at least $\delta_{\rm fpr}$. We start with tests consisting of only one state and grow them via breadth-first search. The goal condition is that the probability the test fails to detect a misaligned agent is less than $\delta_{\rm fpr}$. We can define this probability as
\begin{equation}
    \max_{\pi' \in \Pi'} Pr[\text{$\pi'$ passes test T}] = \max_{\pi' \in \Pi'} \prod_{s \in T} \big( 1- Pr(\pi' \text{ detected at $s$})\big),
\end{equation}
where $\Pi'$ is the set of $\epsilon$ misaligned policies under $R$.

Thus, we perform breadth-first graph search, where the search progressively explored all subsets of states starting with singletons and returns the first subset of states such that $\max_{\pi' \in \Pi'} Pr[\text{$\pi'$ passes test T}] < \delta_{\rm fpr}$. Note that the above test will never fail an $\epsilon$-value aligned agent, since all such agents will never take an action $b$ that satisfies Equation~\eqref{eq:action_query_test} by definition. Thus, we have $\delta_{\rm fnr} = 0$. If we are willing to allow some false negatives (the test is allowed to fail some $\epsilon$-aligned agents), then we can adjust the test by keeping track of all policies that are $\epsilon$-value aligned and computing an analogous probability to that above for false positives.

While the above procedure will work for the simplest of domains it has several fundamental drawbacks: (1) We need to enumerate all policies in $\Pi$ and check whether they are $\epsilon$-value aligned or not, (2) We need to run multiple rollouts from each $\epsilon$-misaligned policy over multiple states to compute $Pr(\pi' \text{ detected at $s$})$ for every state, (3) We have to then run a combinatorial optimization. In comparison, the action query heuristics we discuss in the paper only require solving a single MDP for an optimal policy under the human's reward $R$. However, the heuristics are specifically for testing exact value alignment ($\epsilon = 0$, $\delta_{\rm fpr} = 0$) and do not consider false negatives. Future work should examine how to bridge the gap between these two extremes to see if there is a tractable middle ground that is amenable to high-confidence $\epsilon$-value alignment verification.

\section{Value Alignment Verification Heuristics}\label{app:heuristics}
In this section we discuss the value alignment verification heuristics in more detail. Note that all of the methods above are not guaranteed to verify value alignment and may give false positives. However, all are designed to never give a false negative.

\subsection{Critical State-Action Value Alignment Heuristic}
Prior work by Huang et al. \cite{huang2018establishing}, seeks to build human-agent trust by asking an agent for critical states which are defined as follows:
\begin{equation}
Q_{R}^*(s,\pi_{R}^*(s)) - \frac{1}{|\mathcal{A}|} \sum_{a \in \mathcal{A}} Q_{R}^*(s,a) > t
\end{equation}
for some user-defined $t$. If $t=0$, then all states will be critical states. On the otherhand, for large $t$, none of the states will be critical. Thus, $t$ must be carefully tuned to the scale of the reward function and to the particulars of the MDP. Huang et al. \cite{huang2018establishing} also proposed finding critical states in terms of states with policy entropy below some threshold $t$, but found that state-action value critical states performed better. Futhermore, using entropy would label every state as critical for a deterministic policy. State-action value critical states can also be computed for both deterministic and stochastic policies, thus we only compare against state-action value critical states.

One possible way to use critical states for a value alignment heuristic would be to ask an agent for its critical states and then see if those match the tester's critical states However, this is problematic since reward scale isn't fixed and there are an infinite number of reward functions that lead to the same policy \cite{ng2000algorithms}, so the gap in Q-values can be arbitrarily large. Thus $t$ would have to be carefully constructed and tuned for both the tester and the agent, making this impractical. Instead, we simply calculate the critical states for the tester under a tester-defined $t$ and then test whether the optimal action that the agent being tested would take in the tester's critical state is also optimal under the tester's value function.
 
This results in the following value alignment heuristic:

(1) Find critical states in true MDP for $t\geq 0$.

(2) Query the robot for their action in each critical state and check if this is an optimal action under the tester's reward function.

\subsection{Aligned Reward Polytope Black-Box Heuristic}
For this heuristic we have the tester compute $ARP(R)$ for the tester's reward function $R$, and then find the minimum set of equivalent constraints using linear programming as discussed in Section~\ref{app:halfspace_redundancy_removal}. 
To run a verification test we simply take the set of states corresponding to this minimal set of constraints. For each of these constraints we have
\begin{equation}
\wb^T(\Phi^{(s,a)}_{\pi^*} - \Phi^{(s,b)}_{\pi^*}) > 0
\end{equation}
for all $a \in \arg\max_{a'} Q^*(s,a')$. The test then consists of asking the agent being tested for the action the testee would take in state $s$ and checking if it is optimal under the tester's reward function.

\subsection{SCOT Trajectory-Based Heuristic}
We also adapt the set cover optimal teaching (SCOT) algorithm for value alignment verification \cite{brown2019machine}. As done in the original paper \cite{brown2019machine}, we first compute feature expectations, then we calculate the minimal set of constraints that define the consistent reward set (CRS)  using Corollary~\ref{app-corr:cont_crs}. We then rollout $m$ trajectories using the teacher's policy from each initial state and calculate the CRS of the rollouts using Corollary~\ref{cor:feasibleDemo}. We then run set cover and find the minimum set of rollouts of length $H$ that implicitly covers the CRS.

Given the machine teaching demos from SCOT we mask the actions and ask the agent being tested what action it would take in each state. We then compare this action with the machine teaching action. In particular, we implement this querying the robot agent for an action at each state $s$ and then checking if this action is optimal under the tester's reward function.

\subsection{Computational Complexity}
In terms of complexity, the CS heuristic is the least computationally expensive since it requires only solving for the optimal Q-values at each state and then selecting states with action-value gap larger than $t$. The ARP-bb heuristic is the next most computationally intensive heuristic. It also only requires solving for the optimal policy for a single MDP, but also requires computing $\mathbf{\Delta}$ and removing redundant constraints. If the policy is represented and learned using successor features \cite{barreto2017successor}, then we obtain $\mathbf{\Delta}$ simply by iterating over each state to find optimal and suboptimal actions. Alternatively, given an optimal policy, $\mathbf{\Delta}$ can be efficiently recovered via a vectorized version of policy evaluation, where expected feature vectors are propagated instead of expected values. Removing redundant constraints requires solving a LP. The complexity of this will depend on the number of rows (number of states with unique feature count normal vectors) and columns (number of features) of $\mathbf{\Delta}$. Finally, the SCOT heuristic is the most computationally intensive. It still requires solving one MDP (to get the optimal policy for R), but also requires removing redundant half-space constraints from $\mathbf{\Delta}$ and then running a greedy set-cover approximation.

\section{Case Study Continued}\label{app:case_study}

To illustrate the types of test queries found via value alignment verification, we consider two domains inspired by the AI safety grid worlds~\cite{leike2017ai}. The first domain, \textit{island navigation} is shown in Section~\ref{sec:case_study}. We now discuss another domain inspired by the AI safety gridworlds: lava world. This domain is shown in Figure~\ref{fig:lava_world}. Figure~\ref{subfig:lava_pi} shows the optimal policy under the tester's reward function
\begin{equation}
    R(s) = 50 \cdot \mathbf{1}_{\rm green}(s) - 1 \cdot \mathbf{1}_{\rm white}(s) - 50 \cdot \mathbf{1}_{\rm red}(s),
\end{equation}
where $\mathbf{1}_{\rm color}(s)$ is an indicator feature for the color of the grid cell. Shown in figures~\ref{subfig:lava_pref1} and \ref{subfig:lava_pref2} are the two preference queries generated by ARP-pref. In both cases the query consists of two trajectories (shown in black and orange for visualization), and the agent taking the test must decide which trajectory is preferable (we chose the colors such that the black trajectory is preferable to orange). We see that preference query 1 verifies that the agent would rather move the to terminal state (green) rather than visit white cells. The second preference verifies that the agent would rather visit white cells than red cells, and would rather take an indirect path to the goal state (green) rather than a more direct path that visits red cells. Note that the black trajectory in preference query 2 first goes up, which results in a self transition, then goes left to get out of the lava.
Shown in figures~\ref{subfig:lava_arpbb}, \ref{subfig:lava_scot}, and \ref{subfig:lava_cs} are the query states for ARP-bb, SCOT, and CS heuristics, respectively. In each of these tests the agent being tested is asked what action its policy would take in each of the states marked with a question mark. To pass the test, the agent must respond with an optimal action under the tester's policy in each of these states.
ARP-bb chooses two states where the half-spaces defined by the expected feature counts of following the optimal policy versus taking a suboptimal action and following the optimal policy fully define the ARP. 

\begin{figure}[t]
\centering
\begin{subfigure}[b]{0.24\textwidth}
    \centering
    \includegraphics[width=\textwidth]{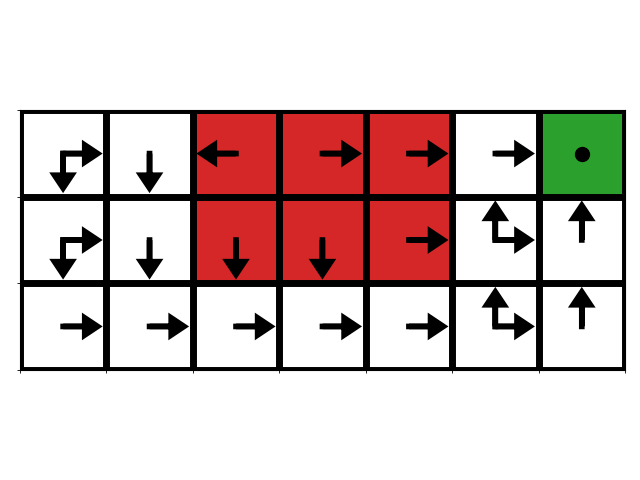}
    \caption{Optimal policy}
    \label{subfig:lava_pi}
\end{subfigure}
\hfill
\begin{subfigure}[b]{0.24\textwidth}
    \centering
    \includegraphics[width=\textwidth]{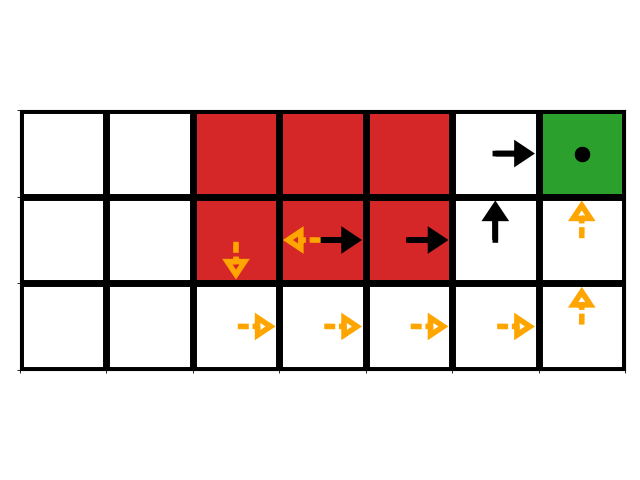}
    \caption{Preference query 1}
    \label{subfig:lava_pref1}
\end{subfigure}
\hfill
\begin{subfigure}[b]{0.24\textwidth}
    \centering
    \includegraphics[width=\textwidth]{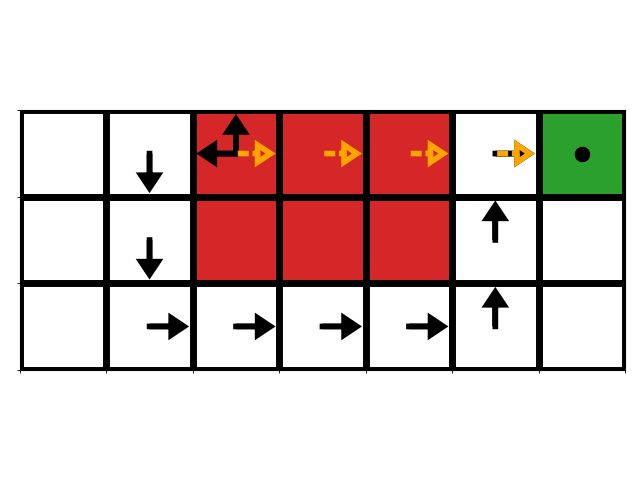}
    \caption{Preference query 2}
    \label{subfig:lava_pref2}
\end{subfigure}
\hfill
\begin{subfigure}[b]{0.24\textwidth}
    \centering
    \includegraphics[width=\textwidth]{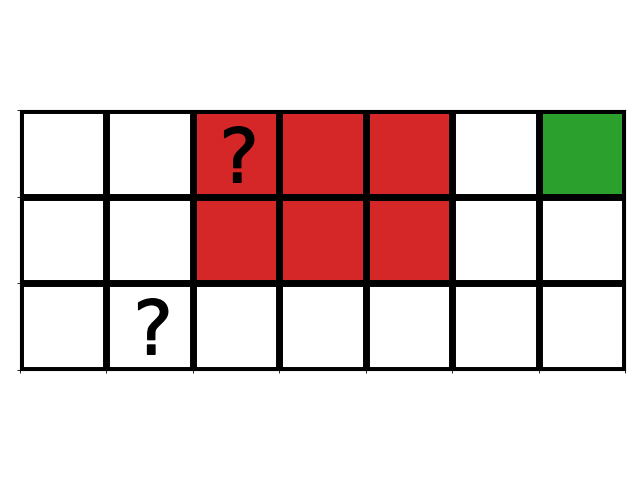}
    \caption{ARP black-box queries}
    \label{subfig:lava_arpbb}
\end{subfigure}

\begin{subfigure}[b]{0.24\textwidth}
    \centering
    \includegraphics[width=\textwidth]{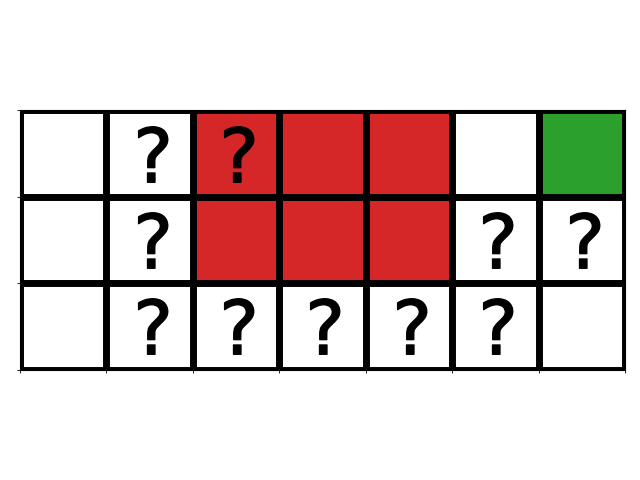}
    \caption{SCOT queries}
    \label{subfig:lava_scot}
\end{subfigure}
\begin{subfigure}[b]{0.24\textwidth}
    \centering
    \includegraphics[width=\textwidth]{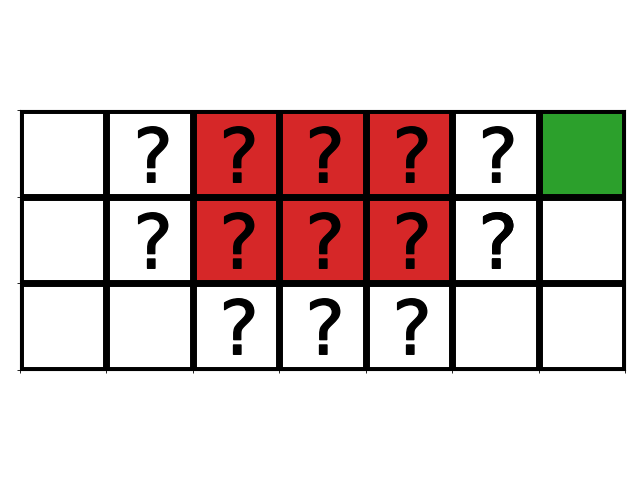}
    \caption{Critical state queries}
    \label{subfig:lava_cs}
\end{subfigure}
\caption{Example value alignment verification tests for the lava world domain.}
\label{fig:lava_world}
\end{figure}

\section{Value Alignment Verification with Idealized Human Tester}\label{app:vav_idealized}

In this appendix we compare the heuristic alignment methods with the exact alignment tests that query for the robot's reward function (ARP-$w$) and query for preferences over trajectories (ARP-pref). Since the tests are designed such that they accurately verify aligned agents, we constructed a suite of grid navigation domains with varying numbers of states and reward features. We generated 50 different misaligned agents by sampling random reward functions and comparing the resulting optimal policies to the optimal policy under a randomly-chosen ground-truth reward function. 
Figure~\ref{fig:tmpname}~(a) and (b) show that for a fixed number of features, the size of the test generated via the critical state heuristic with threshold $t=0.2$ (CS-0.2) scales poorly with the size of the grid world, even though the complexity of the reward function stays constant. The threshold $t$ has a large impact on the performance: small $t$ results in better accuracy at the cost of significantly more queries and larger $t$ results in significantly more false positives. We chose $t=0.2$ to minimize false positives while also attempting to keep the test size small. 
In Figure~\ref{fig:tmpname}~(c) and (d) we plot how the number of constraints grows as the reward function dimension increases and the MDP size is fixed. The plot for ARP-bb shows that the number of constraints grows with the size of the reward weight vector as expected. Conversely, the number of critical states has the undesirable effect of growing with the size of the MDP, regardless of the complexity of the underlying reward function. 

By construction, ARP-w requires only one query (querying for $\wb'$) to achieve perfect accuracy. Using trajectory preferences to define the ARP (ARP-pref) also has perfect accuracy, but requires more queries to the robot. 
SCOT has sample complexity that is lower than the critical state methods, but much higher than querying directly reward function weights since it queries for actions at each state along each machine teaching trajectory. We found empirically that SCOT has nearly perfect accuracy, but occasionally has false positives. Using the ARP inspired heuristic (ARP-bb) has low sample complexity and high accuracy, but sometimes has false positives as expected due to Theorem~\ref{thm:impossibility}.
These results give evidence that the testing method of choice depends on the capability of the robot and the complexity of the environment relative to the robot's reward function. If the robot can report a ground truth reward weight then ARP-w has the best performance. If the robot can only answer trajectory preference queries, then ARP-pref should be used. The heuristics (ARP-bb, SCOT, and CS) have higher query costs and lower accuracy, but are applicable when only given query access to the robot's policy and when the robot may not be perfectly rational. 



\begin{figure}
    \centering
    \begin{subfigure}[b]{0.3\textwidth}
    \centering
    \includegraphics[width=\textwidth]{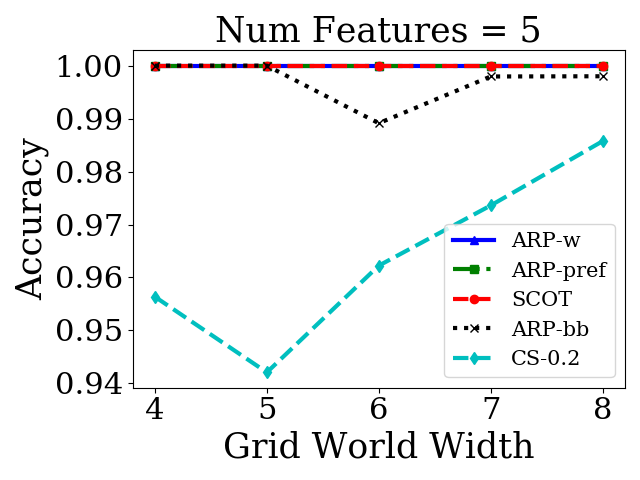}
    \caption{ARP black-box queries}
    \end{subfigure}
    \qquad
    \begin{subfigure}[b]{0.3\textwidth}
    \centering
    \includegraphics[width=\textwidth]{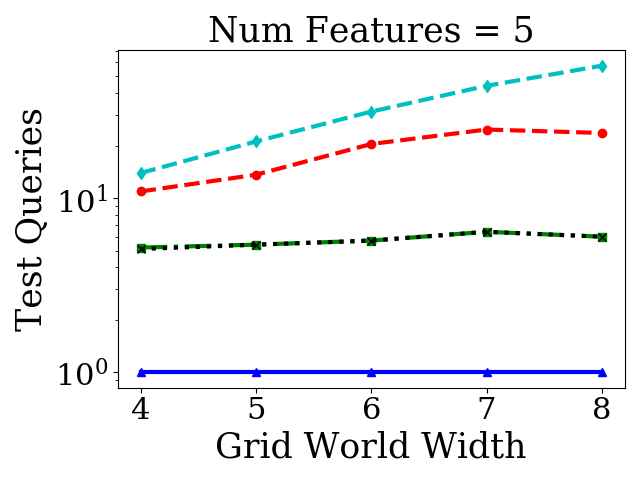}
    \caption{ARP black-box queries}
    \end{subfigure}

    \begin{subfigure}[b]{0.3\textwidth}
    \centering
    \includegraphics[width=\textwidth]{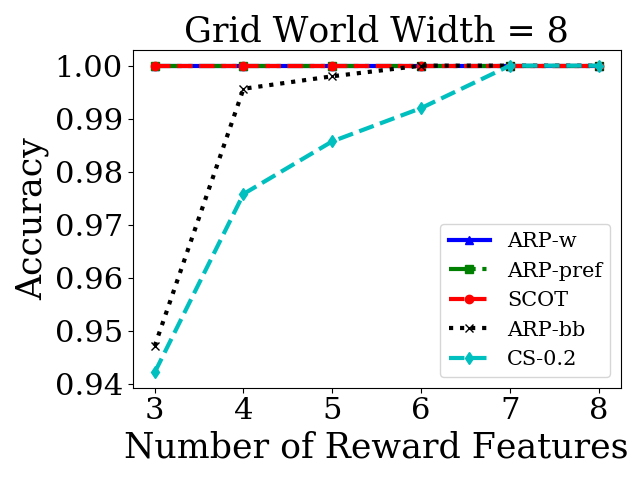}
    \caption{ARP black-box queries}
    \end{subfigure}
    \qquad
    \begin{subfigure}[b]{0.3\textwidth}
    \centering
    \includegraphics[width=\textwidth]{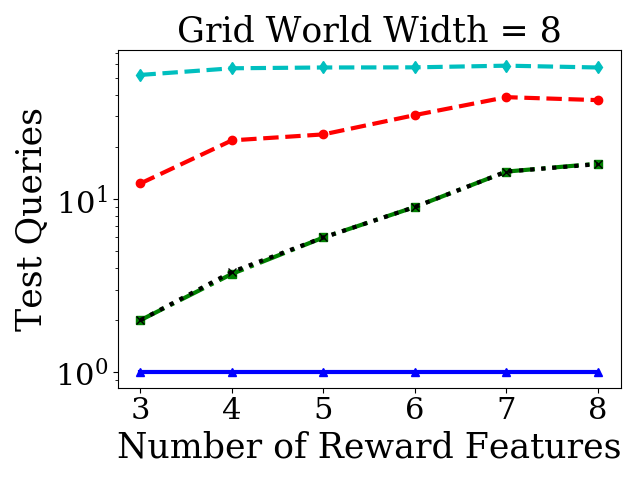}
    \caption{ARP black-box queries}
    \end{subfigure}

    \caption{Queries vs. accuracy (1 - false positive rate) for value alignment testing of misaligned agents. Exact alignment tests (ARP-w and ARP-pref) achieve good efficiency and perfect accuracy.}
    \label{fig:tmpname}
\end{figure}

\section{Details on Value Alignment Verification with Human Tester} \label{app:human-tester}

In this section we fully describe the pipelines used to determine $\epsilon$-alignment in the \emph{implicit human, implicit robot} setting. Each of these pipelines consist of the following steps:

\begin{enumerate}
    \item Preference elicitation, generating a posterior reward distribution and a number of potential test questions and answers
    \item Deduplication of test questions
    \item Filtering questions for $\epsilon$-alignment compatibility using the posterior reward distribution
    \item Removal of redundant questions
\end{enumerate}

Step (1) can be performed using any preference elicitation algorithm which produces a posterior reward distribution, although the questions may be of higher quality if the algorithm can operate over the same linear reward basis the test will operate over. We used the algorithm from \cite{biyik2019asking}, although \cite{pmlr-v87-biyik18a, sadigh2017active} and others with similar structure can be used as easily. Some preference elicitation algorithms that only produce a point estimate of the human's reward function are compatible with some methods of $\epsilon$-filtering, including the one used in the main paper.

Step (2) is necessary because active preference elicitation algorithms sometimes ask the same question multiple times if they believe humans to be noisily reporting their true preference.

Step (3) turns the test from an exact alignment test to an $\epsilon$-alignment test. The goal of this step is to remove questions that $\epsilon$-aligned reward functions will answer wrong. This is true of a question if the value gap between the trajectories under the true reward is less than epsilon $\wb^T \left(\Phi(\xi_a) - \Phi(\xi_b)\right) < \epsilon$. In the implicit human setting, the true reward function is unknown, so this true value gap must be estimated somehow from the posterior distribution. One can achieve this by approximating the true reward $\wb$ with the empirical mean reward $\hat{\wb}:= \mathbb{E}_{\wb \sim P( \cdot | \mathcal{P})}[\wb]$. The MAP reward can also be used, although we did not test this and expect the MAP and mean rewards to be similar.
If one is concerned about the confidence of the estimate, one can instead remove trajectories that are $\epsilon-\delta$ misaligned i.e. those for which $P(\wb^T \left(\Phi(\xi_a) - \Phi(\xi_b)\right) > \epsilon) < \delta$. This ensures that each question in the test has a $1-\delta$ probability of having an value gap of at least $\epsilon$. Empirically we found these two methods had similar results, and so relied on the simpler one.

Step (4) ensures that test is as small as possible. Each question in the test forms a half-space constraint over the possible reward functions. Some of these half-space constraints will be redundant. \cite{brown2019machine} describe a procedure for detecting which half-spaces are redundant by solving a specially constructed linear system of equations. See \ref{app:halfspace_redundancy_removal} for further details. 

The final $\epsilon$-alignment test consists of all the questions asked of the human during step (1) that are not removed by steps (2), (3), or (4). In general there is nothing to suggest that questions asked of the human during a preference elicitation process will make good test questions. In fact, such questions may be stricter than the ones that make up an $\epsilon$-ARP. Questions asked of the human are often between trajectories that are both suboptimal under the human's reward function. The ARP is constructed only using optimal-suboptimal pairs of trajectories, asking the robot only to have the correct preferences in optimal actions, and is agnostic about preferences over suboptimal actions, as the policy will never take those actions. By asking for preferences between suboptimal trajectories, we may be asking the robot to not only have the correct optimal actions at every state, but also the correct rankings between suboptimal actions.

However there are reasons to believe that preference elicitation algorithms that operate over the same linear reward features as the test will ask questions useful for an alignment test. These algorithms share much of the geometry of the ARP. Each answer to a question induces a (potentially soft) half-space constraint over the possible reward function of the human. These algorithms attempt to ask questions that remove the most volume from the posterior reward distribution \cite{sadigh2017active} or have the maximum expected information gain \cite{biyik2019asking}, which intuitively should result in high quality questions.

If one is not satisfied by these arguments, one could use the posterior reward distribution to generate new questions for the test. One could randomly generate test questions and answer those questions using the mean or MAP posterior reward. In practice we found this to have poor performance. With much larger test sizes, the suboptimal-suboptimal trajectory comparisons made the test so strict that nearly no rewards were passed. In future work we would like to generate test questions that will not be too strict by generating optimal trajectories under the posterior reward and comparing them to random suboptimal trajectories.

\section{Experiment Details}\label{app:exp_details}

\subsection{Exact vs Heuristics Grid Domains}
In all grid domains the transition dynamics are deterministic and actions corresponding to movement up, down, left, and right are available at every state. Actions that would lead the agent off of the grid result result in the agent staying in the same state. We ran experiments over different sized grid worlds with different numbers of features. For each grid world size and number of features we generated 50 random MDPs with features placed randomly and with a random ground-truth reward function. We then sampled 50 different reward function weights $w$ from the unit hypersphere. This bounds the Q-values of states, and so allowed us to tune over a bounded interval of $t$ hyperparameters for the critical-action state value alignment heuristic. For each reward we function we computed an optimal policy to create different agents for verification. Duplicate policies were removed.

\subsection{Half-space Normal Vector Redundancy removal}\label{app:halfspace_redundancy_removal}
All experiments (gridworlds and Driver) do duplication and redundancy filtering. Duplicate constraints are detected by computing cosine distance between the halfplane normal vectors. Any normal vectors that are within a small threshold (0.0001) of other normal vectors are deduplicated arbitrarily. Trivial (all-zero) constraints are also removed. There are several known ways to remove redundant constraints~\cite{paulraj2010comparative}. We remove redundant constraints using the exact linear programming method~\cite{paulraj2010comparative}, following the procedure from \citet{brown2019machine} which we will briefly summarize.

A redundant constraint is one that can be removed without changing the interior of the intersection of half-spaces. We can find redundant constraints efficiently using linear programming. To check if a constraint $a^Tx \leq b$ is binding we can remove that constraint and solve the linear program with $\max_x a^Tx$ as the objective. If the optimal solution is still constrained to be less than or equal to $b$ even when the constraint is removed, then the constraint can be removed. However, if the optimal value is greater than $b$ then the constraint is non-redundant. Thus, all redundant constraints can be removed by making one pass through the constraints, where each constraint is immediately removed if redundant.

As an example. Consider Figure~\ref{fig:ARP_toyproblem}. The hatched region on the right is the intersection of half-spaces that makes up the CRS. If we take away the boundaries we get the ARP. Note that there are several half-space constraints that do not tightly define the hatched region and are redundant.

\begin{figure}
    \centering
    \begin{subfigure}[b]{0.3\textwidth}
    \centering
    \includegraphics[width=\textwidth]{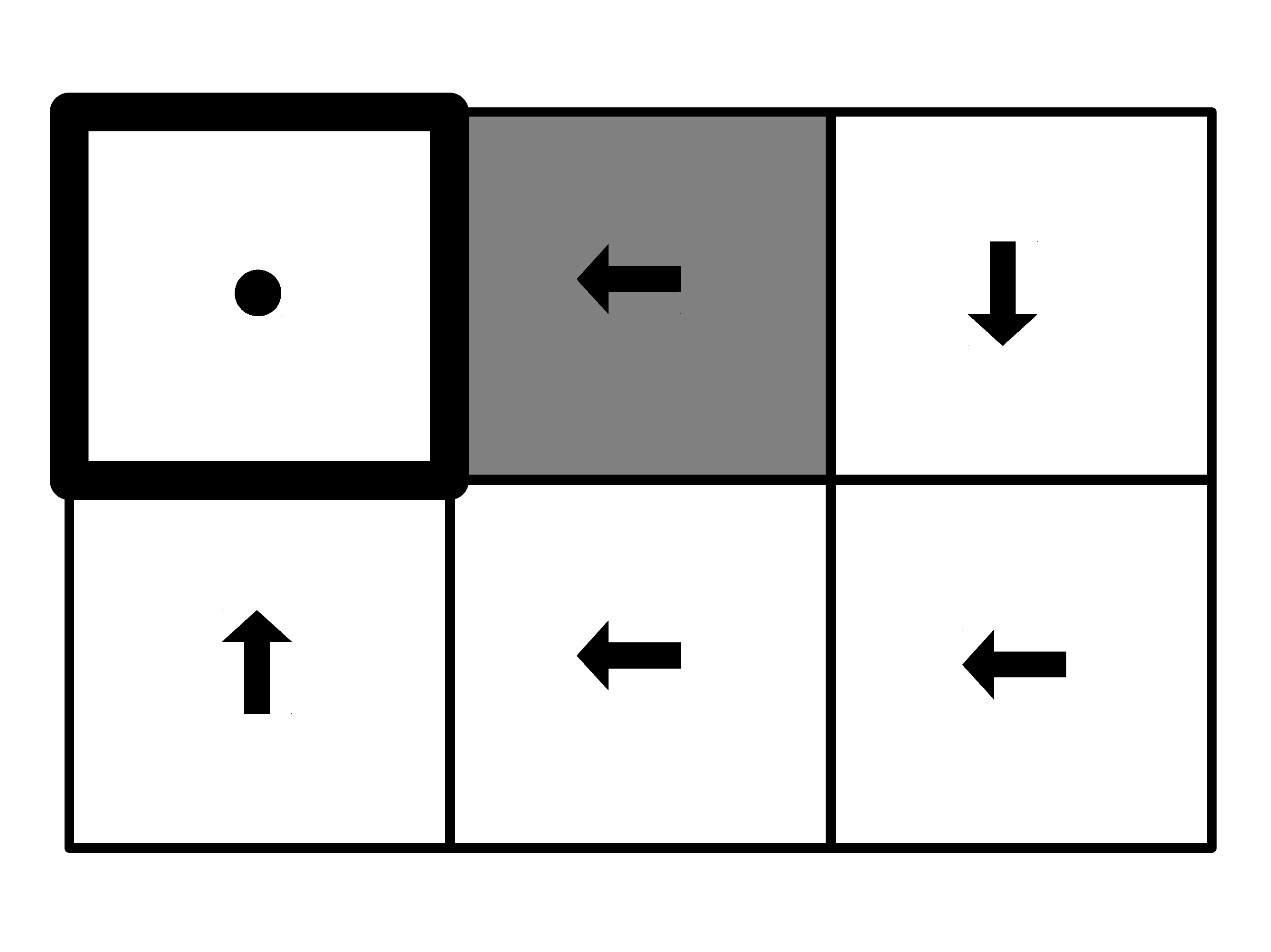}
    \caption{Optimal policy}
    \end{subfigure}
    \qquad
    \begin{subfigure}[b]{0.3\textwidth}
    \centering
    \includegraphics[width=\textwidth]{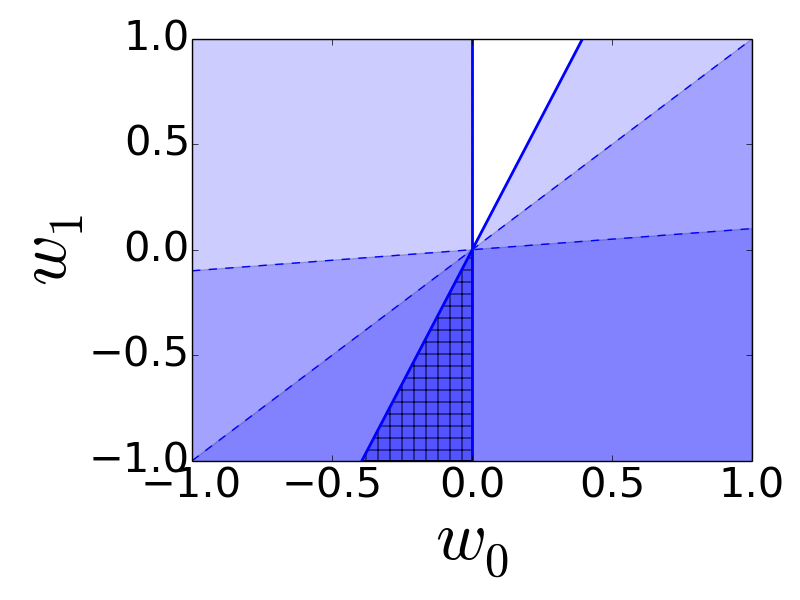}
    \caption{ARP}
    \end{subfigure}
    \caption{Optimal policy and aligned reward polytope (ARP) for a simple gridworld with two features (white and gray) and a linear reward function ($w_0$: weight on white feature, $w_1$: weight on gray feature).}
    \label{fig:ARP_toyproblem}
\end{figure}

\subsection{Simulated Human in the Driver Environment}

In this section we more thoroughly describe the experimental procedure for the driving study in Section~\ref{sec:human} and present the full results, including false positive and false negative rates.

All preference elicitation parameters are as in \cite{biyik2019asking} unless otherwise specified. We run elicitation using strict queries. The information gain criterion was used. Each experiment is replicated ten times. Within each replication, the ground truth reward is randomly sampled from a unit-Gaussian and then normalized to unit length. $M$, the number of rewards sampled from the posterior in order to determine the expected information gain, was set to 100. Each replication terminated after asking different numbers of questions, but each replication contained at least 1000 question-answer pairs.

During test generation, the following combination of parameters were used: $\epsilon \in (0.0, 0.1, \hdots 5.0)$ and the number of simulated human preferences used $n \in (10, 25, 50, 100)$.

Simulated reward functions to evaluate the test were generated in a way to ensure balanced ground truth classes. In all cases, 100000 rewards functions were generated. The reward functions were always generated from a Gaussian distribution with mean equal to the ground truth reward function. The initial variance of this distribution was 1. Ground truth alignment for each test reward was determined by checking for agreement between the test reward and the ground truth reward on all $\epsilon$-compatible questions generated during preference elicitation (>1000) regardless of the value of $n$. If the initial batch of test rewards had between 45\% and 55\% aligned rewards for the given experimental parameters, we deemed the test reward set balanced and proceeded. If the initial set was unbalanced, we adapted the variance of the distribution until it either produced balanced test rewards or became implausibly large or small, in which case we used the last reasonable set of test rewards.

The full false positive, false negative, and accuracy graphs for these experiments are displayed below in Figure~\ref{app:full-sim-results}.

\begin{figure}
    \centering
    \begin{subfigure}[b]{0.4 \textwidth}
        \centering
        \includegraphics[width=\textwidth]{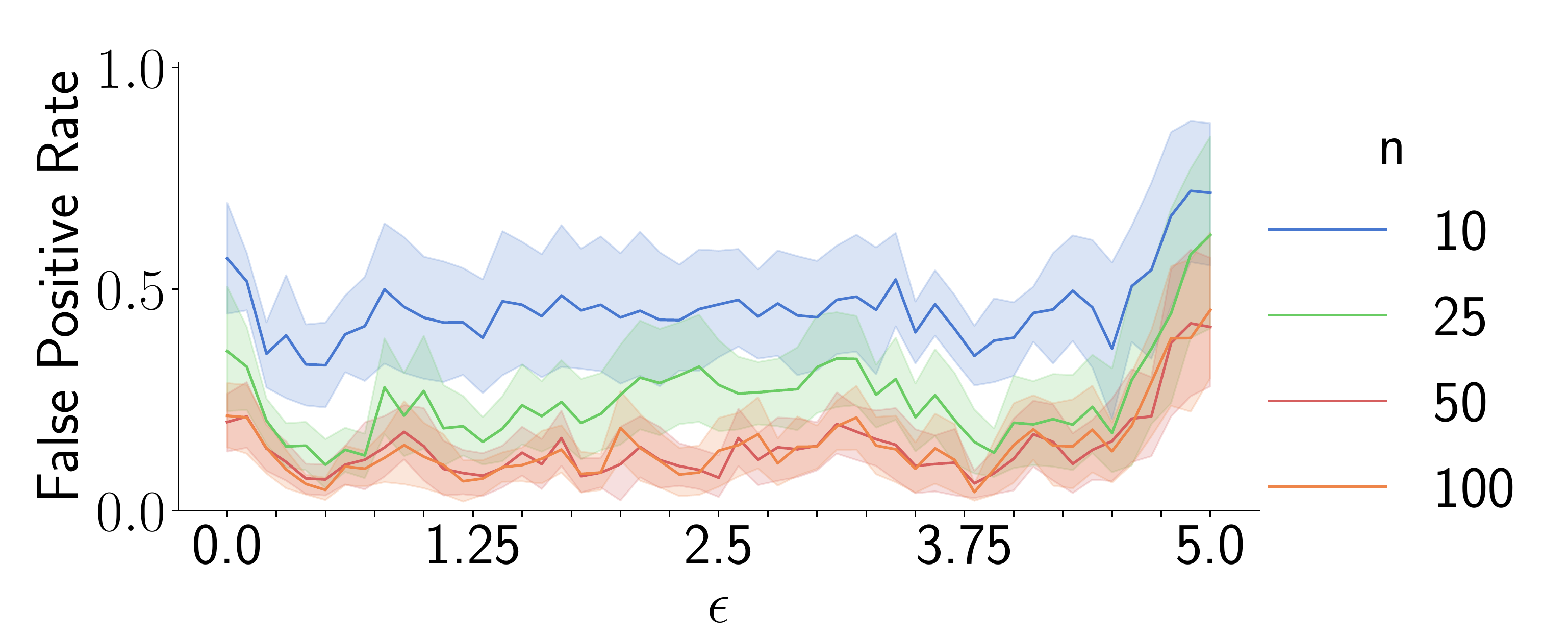}
    \end{subfigure}
    \begin{subfigure}[b]{0.4 \textwidth}
        \centering
        \includegraphics[width=\textwidth]{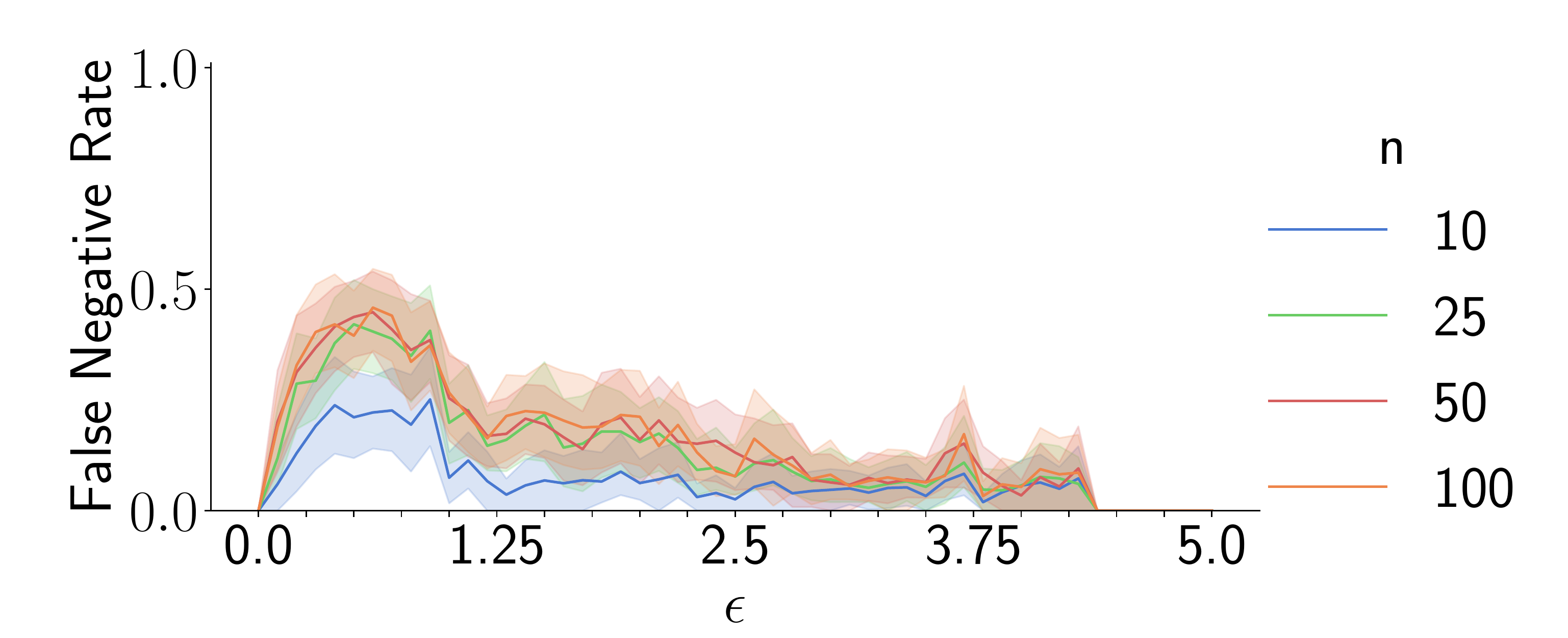}
    \end{subfigure}
    \begin{subfigure}[b]{0.4 \textwidth}
        \centering
        \includegraphics[width=\textwidth]{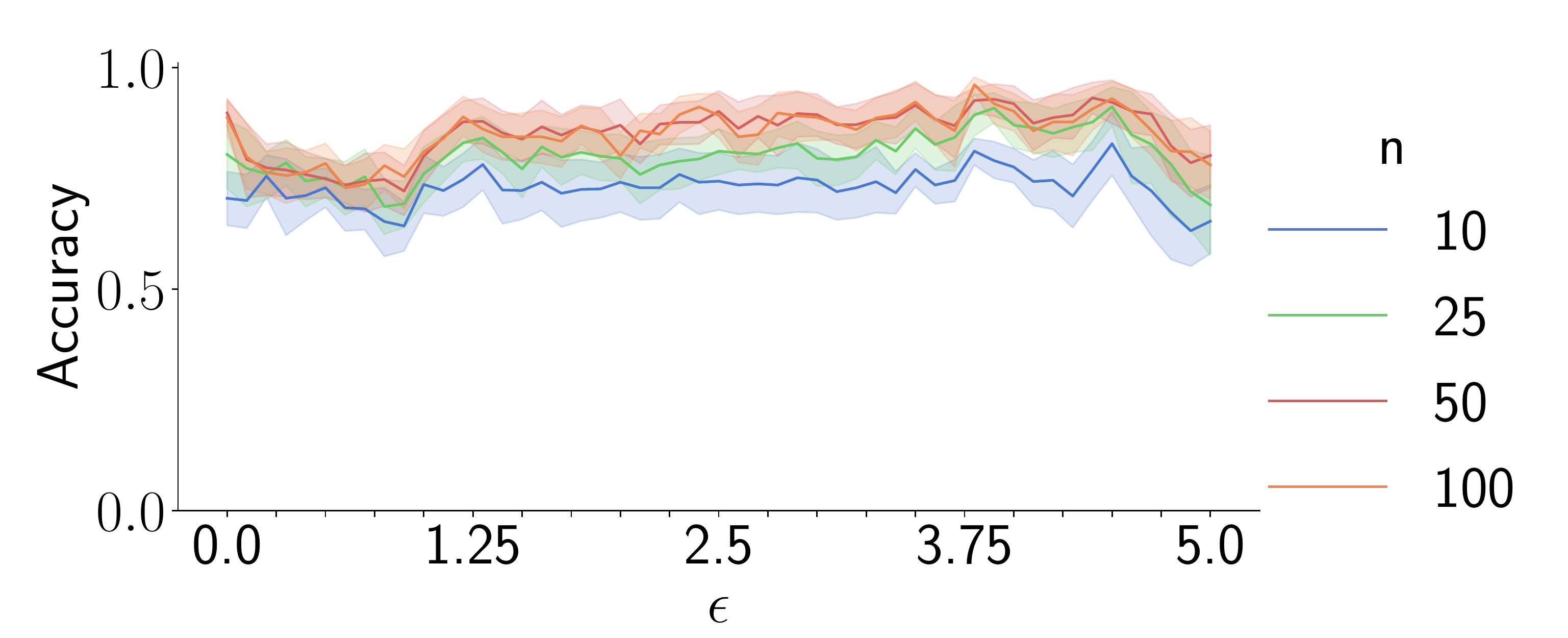}
    \end{subfigure}
    \caption{Performance of tests from simulated humans using mean reward $\epsilon$ filtering.}
    \label{app:full-sim-results}
\end{figure}

\subsection{Human Pilot Study}\label{app:human-exp}
\begin{figure}
\centering
\begin{subfigure}[b]{0.4 \textwidth}
    \centering
    \includegraphics[width=\textwidth]{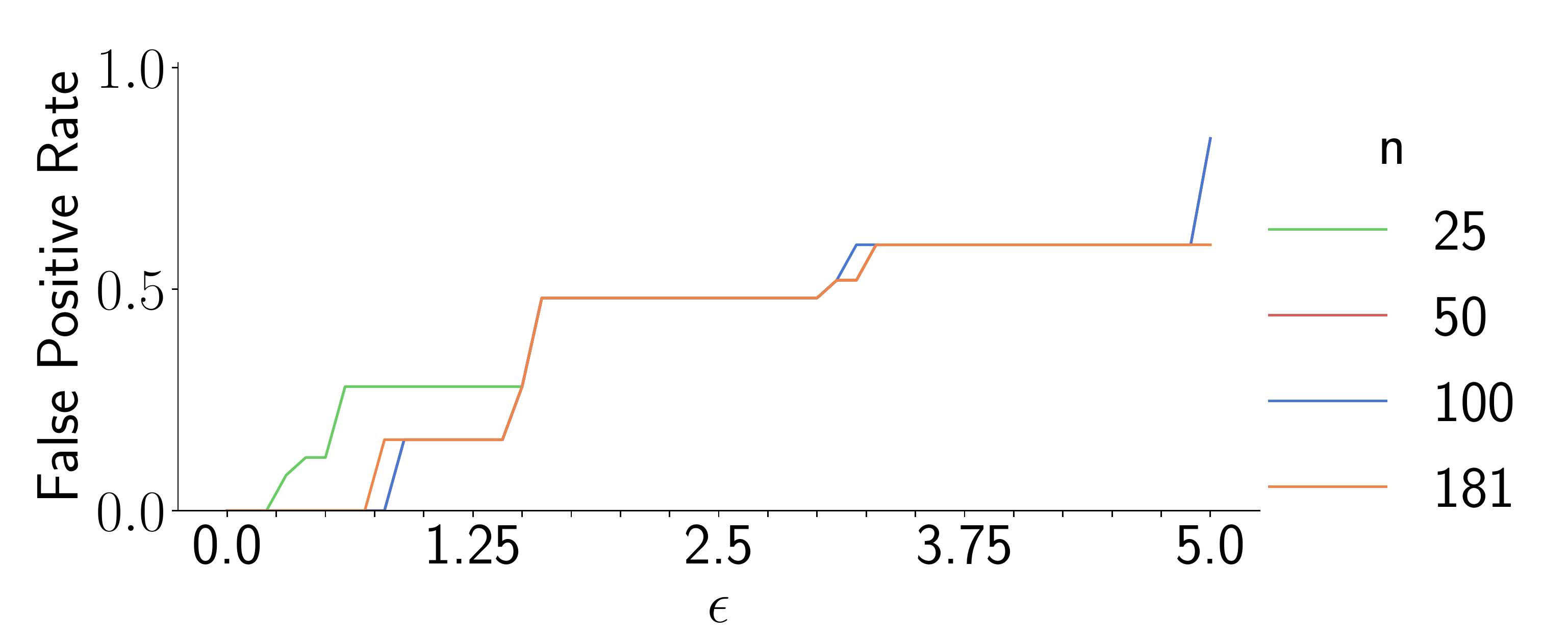}
\end{subfigure}
\begin{subfigure}[b]{0.4 \textwidth}
    \centering
    \includegraphics[width=\textwidth]{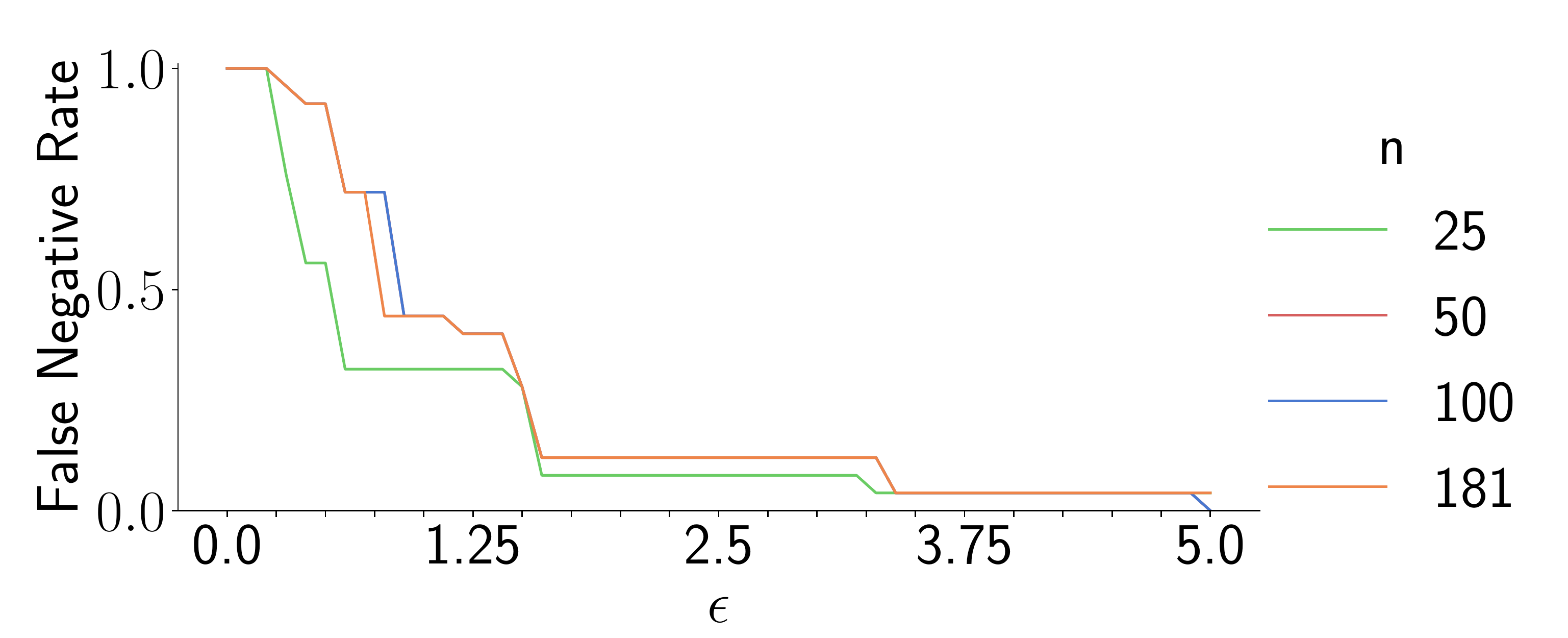}
\end{subfigure}
\begin{subfigure}[b]{0.4 \textwidth}
    \centering
    \includegraphics[width=\textwidth]{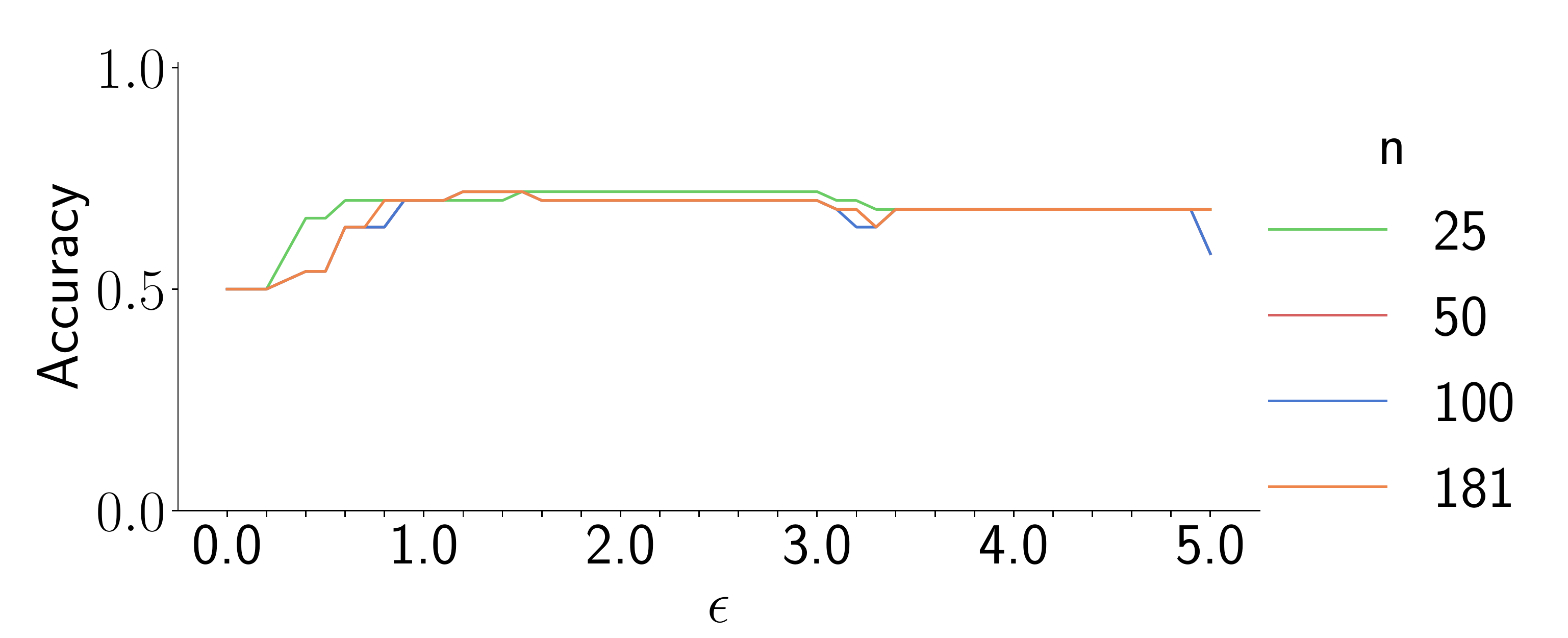}
\end{subfigure}

\caption{Detailed breakdown of mistakes from the human pilot study.}
\end{figure}

We performed a pilot study using real human preferences provided by the authors. The experimental procedure is the same as in the simulated case, except for the generation of test rewards and their labelling. Test rewards were generated from a Gaussian centered at the mean posterior reward with variance 0.5, which was selected after searching for a variance that would provide a balanced test set. Ground truth reward was determined by producing an optimal trajectory as in \cite{biyik2019asking}, manually inspecting the trajectory, and labelling the reward as aligned if the trajectory looked reasonable. This procedure is unjustified, as it does not examine reward functions in states that are hard to reach with a reasonable policy, and so a reward function labeled as aligned may not be aligned everywhere. It serves a reasonable proxy for the pilot study. The results are in figure \ref{app:human-exp}.

As epsilon increases, more of the questions are removed from the test. This necessarily increases the number of positive judgements the test provides, all else being equal. The accuracy initially increases with $\epsilon$ because the test has fewer false negatives as more noise questions are removed. At around $\epsilon=1.0$ most of the aligned agents pass, and any further removal of questions creates more false positives than it removes false negatives, lowering the overall accuracy. The cost of false positives and false negatives are often unequal, and so accuracy may not be the correct metric for your use case. The peak accuracy was 72\%.

\end{document}